\documentclass{style} % For LaTeX2e
\usepackage{xcolor, setspace, array}
\usepackage{enumitem}
\usepackage{float, wrapfig, graphicx, subfigure}
\usepackage{tabularx, booktabs, multirow, makecell, colortbl}
\usepackage{algorithm, algorithmic, enumitem, enumerate}
\usepackage{dsfont, natbib, url, xspace, smile}
\usepackage{subcaption}

\usepackage[most]{tcolorbox}
\usepackage{courier}
\tcbuselibrary{skins,breakable}
\usepackage{xcolor}
\usepackage[table]{xcolor}
\usepackage[scaled=0.92]{inconsolata}

\definecolor{promptbg}{RGB}{245,247,250}
\definecolor{promptborder}{RGB}{120,140,165}

\definecolor{algcommentblue}{RGB}{105,160,205} \newcommand{\algcomment}[1]{\hfill{\color{algcommentblue} $\triangleleft$\,\textit{#1}}}

\newtcolorbox{promptbox}[1]{%
  colback=blue!3, colframe=blue!40, boxrule=0.4pt, arc=1pt,
  left=4pt, right=4pt, top=3pt, bottom=3pt,
  title={\footnotesize\bfseries #1}, coltitle=black, colbacktitle=blue!15,
}

\newtcolorbox{rollout}[1]{%
  colback=blue!3,
  colframe=blue!40,
  boxrule=0.4pt,
  arc=1pt,
  left=4pt,
  right=4pt,
  top=3pt,
  bottom=3pt,
  title={\footnotesize\bfseries #1},
  coltitle=black,
  colbacktitle=blue!15,
  segmentation style={solid, blue!40, line width=0.4pt},
}

\title{An RL View of OPD: Least Square Policy Distillation for Sample-Efficient LLM Reasoning}

\author[1]{Shangzhe Li}
\author[1]{Yuxiao Yang}
\author[2]{Tianrun Yu}
\author[2]{Kaixiang Zhao}
\author[3]{Xiaoyun Wang}
\author[2]{Taylor W. Killian}
\author[1]{Weitong Zhang}

\affiliation[1]{University of North Carolina at Chapel Hill}
\affiliation[2]{Brigham Young University}
\affiliation[3]{NVIDIA}

\begin{document}

\abstract
{We study on-policy distillation (OPD) through the lens of reinforcement learning, establishing a connection between the reverse-KL objective in OPD and KL-regularized policy optimization. Building on this connection, we introduce Least-Square Policy Distillation (LSPD), an RL-inspired framework that brings optimistic exploration and off-policy data reuse from value-based RL into policy distillation. LSPD preserves policy diversity through exploration while improving rollout efficiency by repeatedly learning from previously collected trajectories. Our theoretical analysis connects LSPD to optimistic value-based learning and shows that its idealized formulation achieves a sharp $\tilde{\mathcal O}(\log K)$ regret bound under online exploration. Empirically, LSPD consistently outperforms existing distillation baselines across six mathematical reasoning benchmarks and diverse teacher--student settings, with average gains of $+1.59$ points in Avg@16. Remarkably, through Pass@$k$ evaluations up to $k=64$, we found that LSPD better preserves policy diversity by achieving stronger performance as $k$ grows. Its fully off-policy variant achieves comparable performance to vanilla OPD using only the first $25\%$ of rollout batches. Together, these results provide an RL perspective on OPD that offers both a principled interpretation and a practical route toward more effective and rollout-efficient language model distillation.

\code{https://github.com/UNCSciML/LSPD}
}

\maketitle

% OPD: $\min_{\btheta} \EE_{y \sim \pi_{\btheta}(\cdot | x)} [\log \pi_{\btheta}(y | x) - \log \pi_E(y | x)]$

% gradient: $\EE_{y \sim \pi_{\btheta}(\cdot | x)}  \nabla \log \pi_{\theta}(y | x)\cdot [\log \pi_{\btheta}(y | x) - \log \pi_E(y | x)] - \underbrace{\EE_{y \sim \pi_{\btheta}(\cdot | x)} [\nabla \log \pi_{\btheta}(y | x)]}_{0}$

% gradient: $\EE_{x, y \sim \cD}  \frac{\nabla \pi_{\theta}(y | x)}{\pi_{\cD}(y | x)}\cdot [\log \pi_{\btheta}(y | x) - \log \pi_E(y | x)] - \underbrace{\EE_{y \sim \pi_{\btheta}(\cdot | x)} [\nabla \log \pi_{\btheta}(y | x)]}_{0}$

% LSPD: $\min_{\btheta} \frac12 \EE_{(x, y) \sim \cD} [\log \pi_{\btheta}(y | x) - \log \pi_E (y | x)]^2$

% gradient: $\EE_{(x, y) \sim \cD} \nabla \log \pi_{\btheta}(y | x) \cdot [\log \pi_{\btheta}(y | x) - \log \pi_E (y | x)]$

% \clearpage 
\section{Introduction}

On-policy distillation (OPD)~\citep{lu2025onpolicydistillation} has emerged as a promising approach to transferring and improving the capabilities of large language models (LLMs). Building on knowledge distillation~\citep{hinton2015distilling,kim2016sequence}, OPD trains a student model using token-level teacher supervision on trajectories generated by the student itself. This allows the teacher to provide guidance at prefixes encountered under the student's own generation policy.

While OPD has recently been interpreted as a reinforcement learning (RL) approach~\citep{yang2026learning,lu2025onpolicydistillation} and implemented using policy optimization frameworks such as PPO~\citep{schulman2017proximal} and GRPO~\citep{shao2024deepseekmath}, the algorithmic implications of this formulation remain underexplored, particularly for exploration and data efficiency. These policy-based frameworks are typically implemented in an on-policy manner, requiring time-consuming fresh rollouts throughout training while offering limited support for historical data reuse and sufficient exploration.

On the other hand, KL-regularized RL~\citep{yang2026learning}, with maximum-entropy RL as a special case, has been extensively studied both empirically~\citep{haarnoja2018soft,haarnoja2017reinforcement,schulman2017equivalence} and theoretically~\citep{zhao2025logarithmic,zhao2026sharp}. This literature provides principled algorithms for data-efficient off-policy updates that reuse previously collected experience~\citep{watkins1992q,degris2012offpolicy,munos2016safe}, together with exploration strategies and regularization mechanisms that promote policy diversity and coverage. These advances provide a foundation for exploiting the structure of KL-regularized RL through value-based approaches to policy distillation.

Motivated by the success of value-based KL-regularized RL and the data-efficiency limitations of policy-based OPD implementations, we introduce \emph{Least Square Policy Distillation (LSPD)}, a distillation framework motivated by optimistic KL-regularized value-based RL. Starting from the teacher-induced log-ratio reward, we consider reward estimation using accumulated student trajectories and optimistic policy updates over statistically plausible reward functions. A Lagrangian relaxation of this formulation motivates quadratic matching between student and teacher log-probabilities. For practical token-level training, we combine a robust version of this matching loss with explicit entropy regularization to encourage policy diversity. Crucially, the resulting objective supports optimization on trajectories generated by earlier student policies, enabling both multiple updates per rollout batch and historical data reuse through a replay-buffer variant, \emph{LSPD-RB}. Together, these components provide a practical approach to incorporating exploration-promoting regularization and off-policy learning into policy distillation. To summarize, our main contributions are threefold:
\begin{itemize}[leftmargin=*]
\item \textbf{An RL-inspired policy distillation framework.}
We introduce Least Square Policy Distillation (LSPD), motivated by optimistic KL-regularized policy optimization. LSPD combines robust quadratic matching of student and teacher log-probabilities with explicit entropy regularization, encouraging policy diversity while naturally supporting off-policy optimization. We complement this framework with a theoretical analysis of an idealized optimistic formulation showing that the LSPD enjoys a sharp convergence rate with regret by $\tilde \cO(\log K)$ where $K$ is the update rounds. 

\item \textbf{Effective data reuse for sample-efficient distillation.}
We demonstrate that LSPD effectively leverages both multiple updates per rollout batch and historical trajectories through its replay-buffer variant, LSPD-RB. In our off-policy experiments, LSPD-RB reaches saturated performance in approximately $10$ rollout batches, compared with more than $40$ for LSPD with one update per batch, highlighting the potential of historical data reuse to reduce rollout requirements.

\item \textbf{Improved reasoning performance, diversity, and efficiency.}
Across six benchmarks and three teacher--student settings, LSPD improves Avg@16 and Pass@16 by +1.59 and +1.87 points on average over baselines. LSPD with replay buffer further improves Pass@16 with only 10 training steps, while higher entropy and Pass@$k$ demonstrate better policy diversity and solution coverage.
\end{itemize}

\section{Related Works}

Our work builds on language model distillation and reinforcement learning for LLM post-training. We connect reverse-KL distillation to optimistic KL-regularized RL, introduce a maximum-entropy least-square objective for off-policy data reuse, and establish a sharp theoretical guarantee.

% : language model distillation and reinforcement learning methods for improving language model capabilities.
\vspace{.3em}\noindent \textbf{Language Model Distillation.}
Language model distillation transfers knowledge from a larger teacher model to a smaller student model. Standard approaches typically minimize the forward KL divergence from the teacher to the student using teacher-generated data~\citep{hinton2015distilling,kim2016sequence}. Subsequent work incorporates student-generated on-policy trajectories alongside teacher-generated data~\citep{agarwal2024policy}, or formulates on-policy distillation as an RL problem based on the reverse KL divergence between the student and teacher policies~\citep{gu2024minillm,lu2025onpolicydistillation}. Other studies investigate the effects of different divergence choices~\citep{wu2025rethinking}, interpret language model distillation through the lens of temporal-difference imitation learning~\citep{yu2026language}, and improve its data efficiency~\citep{hsieh2023distilling}. DistiLLM combines skew KL objectives with adaptive off-policy reuse of student-generated responses~\citep{ko2024distillm}, while DistiLLM-2 introduces a contrastive formulation that increases the likelihood of teacher responses and decreases that of student responses~\citep{ko2025distillm2}. Recent extensions of on-policy distillation further address reward extrapolation~\citep{yang2026learning}, introduce entropy-aware training~\citep{jin2026entropy}, and incorporate curriculum-level guidance~\citep{li2026policy}.

% \looseness=-1
\vspace{.3em}\noindent \textbf{Reinforcement Learning for LLM Post-Training.}
Reinforcement learning has been widely used to improve the instruction-following and reasoning capabilities of large language models. \citet{ouyang2022training} introduced reinforcement learning from human feedback for aligning language models with human preferences, while more recent work employs verifiable, rule-based rewards to improve mathematical and general reasoning capabilities~\citep{shao2024deepseekmath,guo2025deepseek,yu2026dapo}. Complementary to online RL, Identity Preference Optimization (IPO) learns directly from pairwise preferences using a squared loss on differences of policy-to-reference log-likelihood ratios~\citep{azar2024general}, whereas our quadratic objective matches student and teacher log-probabilities. A growing theoretical literature has established sharp performance guarantees~\citep{zhao2026sharp} and logarithmic regret bounds~\citep{zhao2025logarithmic} for KL-regularized RL.
\section{Preliminaries}
\label{sec:prelim}
We formulate language model post-training from an RL view: at the sequence level, the post-training of LLM can be viewed as a contextual bandit. In particular, at position $t$, we denote the prefix $\xb_{<t} = (\qb, \yb_{<t}) \in \cX$ using the query $\qb$ from the dataset $\cD$ and the tokens generated by the model $\yb_{<t}$. The model then generates the next token $y_t \in \cY$ from the policy $\pi(\cdot | \xb)$. Through this process, we define the reward $R: \mathcal{X}\times\mathcal{Y}\rightarrow\mathbb{R}$ and it's class $\cR \ni R$ which the policy seeks to maximize.

\vspace{.3em}\noindent \textbf{Maximum-Entropy and KL-Regularized Reinforcement Learning.} Instead of greedy maximizing the reward by $\argmax_{\pi} \EE_{\pi}[R(\qb, \yb)]$, maximum-entropy reinforcement learning~\citep{maxentirl2008,haarnoja2018soft}, or generally, the KL regularized reinforcement learning~\cite{zhao2025logarithmic}, regularize the policy optimization with the KL divergence with the objective
\begin{align}
\textstyle{\argmax_{\pi}} \mathbb{E}_{\yb\sim\pi(\cdot|\qb)}[R(\qb,\yb)] - \eta^{-1} D_{\mathrm{KL}}\!\left(\pi(\cdot|\qb) \,\|\, \pi^{\mathrm{ref}}(\cdot|\qb) \right).
\label{eq:max-ent-rl}
\end{align}
where the $\pi^{\mathrm{ref}}$ denotes a fixed reference policy. When $\pi^{\mathrm{ref}}$ is the uniform policy on the action set $\mathcal{A}$, it can be verified that $D_{\mathrm{KL}}\!\left(\pi(\cdot|\qb) \,\|\, \pi^{\mathrm{ref}}(\cdot|\qb) \right)$ becomes the negative entropy ${-}\mathcal{H}\!\left(\pi(\cdot|\qb)\right)$ and Eq.~\ref{eq:max-ent-rl} becomes the typical setting of maximum entropy RL $\argmax_{\pi} \mathbb{E}_{\yb\sim\pi(\cdot|\qb)}[R(\qb,\yb)] + \eta^{-1} \mathcal{H}\!\left(\pi(\cdot|\qb)\right)$.

\vspace{.3em}\noindent \textbf{On-Policy Distillation.}
On-policy distillation (OPD,~\citealt{lu2025onpolicydistillation}) minimizes the token-level reverse KL divergence between the teacher policy $\pi^E$ with the student policy $\pi$ by
\begin{align}
\mathcal{L}_{\mathrm{OPD}}(\pi) =
\mathbb{E} \left[\sum_{t=1}^{T}
D_{\mathrm{KL}}\!\left(\pi(\cdot|\xb_{<t})\| \pi^E(\cdot|\xb_{<t})\right)\right] = \mathbb{E} \left[\sum_{t=1}^{T} 
\EE_{y_t \sim \pi(\cdot|\xb_{<t})} \log \frac{\pi(y_t|\xb_{<t})}{\pi^E(y_t|\xb_{<t})}\right],
\label{eq:opd-prelim}
\end{align}
where the expectation is taken over prefix $\xb_{<t}$ is sampled from student policy $\pi$ from some query $\qb$, which is dubbed as \emph{on-policy} since it requires student's rollout. Notably, the current common implementation of OPD leverages the on-policy optimization framework like PPO~\citep{schulman2017proximal} by taking the policy gradient (and additional clippings in PPO) by  
\begin{align}
\nabla \mathcal{L}_{\mathrm{OPD}}(\pi)\! =\!
 \mathbb{E} \left[\sum_{t=1}^{T} 
\nabla \log \pi(y_t | \xb_{<t}) \cdot \log \frac{\pi^\perp(y_t|\xb_{<t})}{\pi^E(y_t|\xb_{<t})}\right]\!=\! \mathbb{E} \left[\sum_{t=1}^{T} 
\frac{\nabla \pi(y_t | \xb_{<t})}{\pi^\perp(y_t | \xb_{<t})} \cdot \log \frac{\pi^\perp(y_t|\xb_{<t})}{\pi^E(y_t|\xb_{<t})}\right],
\notag 
\end{align}
where $\pi^\perp$ denotes the stopping gradient and $\log \frac{\pi^\perp(y_t|\xb_{<t})}{\pi^E(y_t|\xb_{<t})}$ is the advantage function for PPO process.

\section{Methodology and Analysis}

\subsection{On-Policy Distillation as KL-Regularized Reinforcement Learning}

We reinterpret the reverse-KL objective in Eq.~\ref{eq:opd-prelim} as token-level policy optimization with a teacher-induced reward. Following the notation in Section~\ref{sec:prelim}, we fix an arbitrary reference policy $\pi^{\mathrm{ref}}$ and define $R(\xb_{<t},y_t)=\log\!\left(\pi^E(y_t|\xb_{<t})/\pi^{\mathrm{ref}}(y_t|\xb_{<t})\right)$, assuming that the teacher and reference policies assign positive probability to tokens in the student's support at each prefix. Then, the OPD objective in Eq.~\ref{eq:opd-prelim} becomes
\begin{align}
\argmin_{\pi}\mathcal{L}_{\mathrm{OPD}}(\pi)
&= \argmin_{\pi} \mathbb{E}\left[
\sum_{t=1}^{T}\log\frac{\pi(y_t|\xb_{<t})}{\pi^E(y_t|\xb_{<t})}
\right] \nonumber\\
&= \argmin_{\pi} \mathbb{E}\left[
\sum_{t=1}^{T}\left(
\log\frac{\pi(y_t|\xb_{<t})}{\pi^{\mathrm{ref}}(y_t|\xb_{<t})}
-\log\frac{\pi^E(y_t|\xb_{<t})}{\pi^{\mathrm{ref}}(y_t|\xb_{<t})}
\right)\right] \nonumber\\
&= \argmax_{\pi} \mathbb{E}\left[
\sum_{t=1}^{T}\left(
\EE_{y_t\sim\pi(\cdot|\xb_{<t})}R(\xb_{<t},y_t)
-D_{\mathrm{KL}}\!\left(\pi(\cdot|\xb_{<t})\|\pi^{\mathrm{ref}}(\cdot|\xb_{<t})\right)
\right)\right].
\label{eq:opd-kl-rl}
\end{align}
Here, expectations are taken over $\qb\sim\cD$ and autoregressive generation from $\pi$, with $\xb_{<t}=(\qb,\yb_{<t})$. Thus, OPD is the token-level counterpart of the KL-regularized policy optimization problem in Eq.~\ref{eq:max-ent-rl} with $\eta=1$. The reward $R$ is fixed with respect to the student and measures the teacher's next-token log-likelihood relative to the reference, while the KL penalty discourages deviations from that reference at each prefix. Importantly, this is an exact reformulation of OPD rather than an additional regularization term. Choosing $\pi^{\mathrm{ref}}$ as a fixed initial model yields the reward--KL structure commonly used in RLHF~\citep{ouyang2022training,rafailov2023direct}, with the reward specified directly by the teacher-to-reference likelihood ratio. Prior work~\citep{yang2026learning} has also noted the connection between OPD and KL-regularized RL.

% \clearpage
% Given a teacher policy $\pi^E$ and a student policy $\pi$, on-policy distillation (OPD) trains the student on trajectories generated by its current policy and queries the teacher at the resulting student-visited prefixes~\citep{lu2025onpolicydistillation}. Let $d_t^\pi$ denote the distribution of the token-level state $s_t=(x,y_{<t})$ induced by $x\sim\mathcal{D}$ and autoregressive generation from $\pi$. The reverse-KL OPD objective is

% By the chain rule of KL divergence, this token-level objective is equivalent to $\mathbb{E}_{x\sim\mathcal{D}}\!\left[D_{\mathrm{KL}}\!\left(\pi(y_{1:T}|x)\,\|\,\pi^E(y_{1:T}|x)\right)\right]$. Thus, ``on-policy'' refers to evaluating teacher feedback under the prefix distribution induced by the student; the teacher provides token distributions at these prefixes but does not generate the training trajectories.

\subsection{Least Square Policy Distillation}

\paragraph{Optimistic KL-regularized Update.} Given the KL-regularized objective introduced in the previous sub-section, the next question is how to solve the resulting policy optimization problem. A straightforward approach is to apply PPO \citep{schulman2017proximal}, using advantage estimation together with the clipped surrogate objective. However, PPO is not an optimistic optimization method: its updates are driven by the current reward estimates and do not explicitly account for uncertainty in a way that encourages exploration toward potentially better tokens. Motivated by prior works on KL-regularized reinforcement learning \citep{zhao2026sharp,zhao2025logarithmic}, which leverage optimistic reward estimation for policy optimization, we instead consider a more direct approach that admits an explicit, ``closed-form'' iterative update of the conditional token distribution. Specifically, after collecting rollouts through iteration $k$, we estimate the induced token-level reward using their prefix--token pairs:
\begin{align}
\hat r_k
&\in
\textstyle{\argmin_{r \in \mathcal{R}}}
\textstyle{\sum_{i=1}^{k}}
\EE_{\qb\sim\cD,\,\yb\sim\pi_i(\cdot|\qb)}
\left[
\textstyle{\sum_{t=1}^{T}}
\left(
 r(\xb_{<t},y_t)-R(\xb_{<t},y_t)
\right)^2
\right]
\nonumber\\
&=
\textstyle{\argmin_{r \in \mathcal{R}}}
\textstyle{\sum_{i=1}^{k}}
\EE_{\qb\sim\cD,\,\yb\sim\pi_i(\cdot|\qb)}
\left[
\textstyle{\sum_{t=1}^{T}}
\left(
 r(\xb_{<t},y_t)
-
\log
\frac{\pi^E(y_t|\xb_{<t})}
{\pi^{\mathrm{ref}}(y_t|\xb_{<t})}
\right)^2
\right].
\label{eqn:regression}
\end{align}

For a stored rollout $(\qb_i,\yb_i)$, let $y_{i,t}$ denote its token at position $t$ and $\xb_{i,<t}=(\qb_i,\yb_{i,<t})$ its prefix. We then construct a least-squares confidence set containing reward functions that remain statistically consistent with the accumulated prefix--token pairs:
\begin{align}
\mathcal{C}_k
=
\left\{
 r\in\mathcal{R} :
\textstyle{\sum_{i=1}^{k}\sum_{t=1}^{T}}
\left[r(\xb_{i,<t},y_{i,t})-R(\xb_{i,<t},y_{i,t})\right]^2
+\lambda
\leq
\beta^2
\right\}.
\label{eqn:uncertainty}
\end{align}

Here, $\beta$ denotes the confidence radius and $\lambda>0$ is a regularization parameter. We directly find the pointwise optimistic token-level reward
$r_k^+(\xb_{<t},y_t)
=
\sup_{r\in\mathcal{C}_k}r(\xb_{<t},y_t)$.
At each fixed prefix, we then use the closed-form KL-regularized token update
\[
\pi_{k+1}(y_t|\xb_{<t})\propto
\pi^{\mathrm{ref}}(y_t|\xb_{<t})
\exp\!\left(\eta r_k^+(\xb_{<t},y_t)\right).
\]
This update optimizes the conditional token distribution while holding the prefix fixed. Thus, each iteration uses stored prefix--token pairs to estimate the induced reward and constructs optimistic next-token distributions from the most favorable pointwise reward estimates compatible with the collected data. Algorithm~\ref{alg:rpd-theory} summarizes the corresponding sequence-level theoretical idealization.

\paragraph{Lagrangian Relaxation.}
The pointwise optimization $r_k^+(\xb_{<t},y_t)=\sup_{r\in\mathcal{C}_k}r(\xb_{<t},y_t)$ can be reformulated through a Lagrangian relaxation. Ignoring terms that are constant with respect to $r$, for each fixed prefix $\xb_{<t}$ and token $y_t\in\cY$, we obtain
\begin{align}
\arg\max_{r\in\mathcal{R}}
 r(\xb_{<t},y_t)
-
\mu
\textstyle{\sum_{i=1}^{k}\sum_{u=1}^{T}}
\left[
 r(\xb_{i,<u},y_{i,u})-R(\xb_{i,<u},y_{i,u})
\right]^2,
\end{align}
where $\mu\geq 0$ denotes the Lagrange multiplier and $u$ indexes token positions in the stored rollouts. We further consider a centered reward function class 
$
\mathcal{R}
=
\left\{
 r_\pi(\xb_{<t},y_t)
=
\log\!\left({\pi(y_t|\xb_{<t})}
/{\pi^{\mathrm{ref}}(y_t|\xb_{<t})}\right)
:
\pi\in\Pi
\right\}
$, which is parameterized by token-level policies. Since $\pi^{\mathrm{ref}}$ is fixed, the pointwise optimistic optimization at each prefix can therefore be equivalently written as
\begin{align}
\forall y_t\in\cY:
\arg\max_{\pi\in\Pi}
\log\pi(y_t|\xb_{<t})
-
\mu
\sum_{i=1}^{k}\sum_{u=1}^{T}
\left(
\log\pi(y_{i,u}|\xb_{i,<u})
-
\log\pi^E(y_{i,u}|\xb_{i,<u})
\right)^2
.
\label{eqn:rpd}
\end{align}
In this way, we obtain an optimistic token-level distillation formulation motivated by the reverse-KL objective of OPD \citep{lu2025onpolicydistillation}. Importantly, the resulting objective is off-policy, allowing the algorithm to leverage prefix--token pairs from historical rollouts for policy updates rather than restricting training to samples generated by the current policy. Moreover, the mechanism through which optimism is incorporated into the objective is closely related to maximum-entropy objectives studied in the reinforcement learning literature \citep{maxentirl2008,haarnoja2018soft} and, more recently, in the LLM post-training literature \citep{xie2025exploratory}.

\subsection{Practical Implementation}

In practice we are given a query dataset $\mathcal{D}$, a teacher policy $\pi^E$, and a student policy $\pi$. The student generates rollouts $\yb$, conditioned on queries $\qb$ sampled from $\mathcal{D}$, yielding prefix--token pairs $(\xb_{<t},y_t)$. Building on Eq.~\ref{eqn:rpd}, we use a practical surrogate that combines robust token-level log-probability matching with an entropy bonus in place of the pointwise optimistic term:
\begin{align}
\mathcal{J}(\pi)
=
\mathbb{E}_{(\qb,\yb)\sim\mathcal{B}}
\left[
\tfrac{1}{T}
\textstyle{\sum_{t=1}^{T}}
\left(
\left(\log \pi(y_t|\xb_{<t})
-
\log \pi^E(y_t|\xb_{<t})\right)^2
-
\mu^{-1}\,
\mathcal{H}\!\left(\pi(\cdot|\xb_{<t})\right)
\right)
\right],
\label{eqn:rpd-practical}
\end{align}
where $\mathcal{B}$ denotes a replay buffer that stores query-response pairs, $\mathcal{H}\!\left(\pi(\cdot|\xb_{<t})\right)$ denotes the entropy of the current student's next-token distribution at a stored prefix.
% , and
% \begin{align}
% \Delta_t(\pi,\pi^E)
% =
% \log \pi(y_t|\xb_{<t})
% -
% \log \pi^E(y_t|\xb_{<t})
% \end{align}
% denotes the token-level log-probability difference between the student and teacher policies. 

To improve training stability and reduce the influence of excessively large log-probability differences, in practical implementations, we replace the standard quadratic loss with a Huber-type robust penalty $\psi(\log \pi(y_t|\xb_{<t})
- \log \pi^E(y_t|\xb_{<t}))$, detailed in Appendix~\ref{sec:huber}. Notably, the objective in Eq.~\ref{eqn:rpd-practical} supports both on-policy and off-policy optimization. The on-policy (or semi-on-policy) variant uses freshly generated responses from the current policy $\pi$ for one or multiple optimization updates, whereas the off-policy variant additionally reuses responses generated by earlier policies through a replay buffer. This flexibility allows rollout data to be efficiently reused across multiple updates. In the following experiments, we refer to the former as Least Square Policy Distillation (LSPD) and the fully off-policy variant as Least Square Policy Distillation with Replay Buffer (LSPD-RB). The detailed practical implementation is summarized in Algorithm~\ref{alg:rpd}.

% In the preamble
\begin{algorithm}[t]
\caption{Least Square Policy Distillation}
\label{alg:rpd}
\begin{algorithmic}[1]
\REQUIRE Student policy $\pi$, teacher policy $\pi^E$, replay buffer $\mathcal{B}$, batch size $B$, number of updates $N$.
\STATE Initialize $\mathcal{B}\leftarrow\emptyset$.
\FOR{$k=1,\ldots,K$}
    \STATE Sample queries $\{\qb_b\}_{b=1}^B$ and responses
    $\{\yb_{b}\}\sim\pi_k(\cdot|\qb_b)$.
    \algcomment{On-policy rollout collection}

    \STATE $\mathcal{B}\leftarrow
    \mathcal{B}\cup\{(\qb_b,\yb_{b})\}$.

    \FOR{$i=1,\ldots,N$}
        \STATE Sample a minibatch from $\mathcal{B}$ and update the student
        using Eq.~\ref{eqn:rpd-practical}.
    \ENDFOR

    \IF{use LSPD}
        \STATE $\mathcal{B}\leftarrow\emptyset$.
        \algcomment{No replay in LSPD; LSPD-RB reuses historical trajectories.}
    \ENDIF
\ENDFOR
\end{algorithmic}
\end{algorithm}

\section{Experiments}

We evaluate the empirical performance of our proposed method against a range of post-training baselines on mathematical reasoning tasks. We describe the experimental setup in Section~\ref{sec:exp-setup}, present the main results in Section~\ref{sec:emp-result}, investigate off-policy learning and ablation studies in Section~\ref{sec:off-policy}. Results on out-of-domain evaluation and more ablation studies are shown in Appendix~\ref{sec:addition-exp}.

\subsection{Experimental Setup}
\label{sec:exp-setup}

\paragraph{Model Setup.} We consider three distillation setups based on the Qwen3 model family~\citep{yang2025qwen3}. Specifically, we use Qwen3-8B as the teacher and Qwen3-4B-Base as the student, Qwen3-4B as the teacher and Qwen3-1.7B-Base as the student, and Qwen3-1.7B as the teacher and Qwen3-0.6B-Base as the student. For all teacher models, we disable thinking mode and use standard non-thinking model configuration.

\vspace{.3em} \noindent \textbf{Dataset and Evaluation Setup.} For all experiments, we use DAPO-Math-17K \citep{yu2026dapo} as the training dataset for mathematical reasoning. For evaluation, we consider a diverse set of challenging mathematical reasoning benchmarks spanning different difficulty levels, including MATH-500 \citep{hendrycks2021measuring}, Minerva \citep{lewkowycz2022solving}, Olympiad-Bench \citep{he2024olympiadbench}, AMC23 \citep{maa2023amc}, AIME24 \citep{aime24}, and AIME25 \citep{aime25}. We report the average score over 16 generations (Avg@16)  as the final result for all baselines and benchmarks in Table~\ref{tab:main_results}.

\vspace{.3em} \noindent \textbf{Generation Configurations.} During training, we use top-$p$ sampling with $p=1.0$, a temperature of $1.0$, a maximum rollout length of 7168 tokens, and four rollouts per query. During evaluation, we use top-$p$ sampling with $p=0.95$, a temperature of $0.7$, and the same maximum rollout length of 7168 tokens. For both training and evaluation, we use the prompt template described in Appendix~\ref{sec:prompt-template}.

\vspace{.3em} \noindent \textbf{Baselines.}
We compare LSPD and LSPD-RB against several language model distillation baselines: Knowledge Distillation (KD) \citep{hinton2015distilling,kim2016sequence}, On-Policy Distillation (OPD) \citep{lu2025onpolicydistillation}, and Entropy-Aware On-Policy Distillation (EOPD) \citep{jin2026entropy}. KD performs off-policy distillation on teacher-generated trajectories using cross-entropy supervision together with forward KL divergence between the student and teacher distributions. In contrast, OPD trains on student-generated trajectories using the reverse-KL objective introduced in Section~\ref{sec:prelim}. EOPD further augments OPD with forward KL regularization at high-entropy teacher token positions to better preserve teacher uncertainty and output diversity.

\subsection{Main Results}
\label{sec:emp-result}
\paragraph{Results on Mathematical Reasoning Capabilities.}
Table~\ref{tab:main_results} shows that LSPD consistently delivers strong mathematical reasoning performance across model scales and benchmarks. Averaged over all 18 model--benchmark combinations, LSPD achieves an Avg@16 of 31.60 and Pass@16 of 53.30, outperforming the strongest baseline, EOPD, by +0.91 and +0.85 points, respectively, and standard OPD by +1.99 and +2.27 points. Notably, LSPD-RB achieves a comparable Avg@16 of 31.51 while further improving Pass@16 to 54.86, exceeding EOPD by +0.82 Avg@16 and +2.42 Pass@16, and OPD by +1.89 and +3.84 points, respectively. When all five methods are considered, LSPD attains the best or tied-best Avg@16 on 11 of 18 settings and Pass@16 on 9 of 18 settings, while remaining top-2 on 17 of 18 and 16 of 18 settings, respectively. More importantly, considering LSPD and LSPD-RB together, at least one of the two achieves the best or tied-best result on 17 of 18 settings for both Avg@16 and Pass@16, and ranks second or tied-second in each of the two remaining cases. Thus, the better-performing LSPD variant is top-2 across all 36 metric--setting combinations, demonstrating robust improvements in both average-generation accuracy and multi-sample solution coverage. Remarkably, LSPD-RB achieves these results using only 10 training steps, whereas the baselines require more than 40 steps to converge stably, highlighting the substantial sampling efficiency enabled by off-policy reuse.

\begin{table*}[t]
    \centering
    \caption{
    \textbf{Main Results on Mathematical Reasoning.}
    We report Avg@16 and Pass@16 performance of our proposed LSPD
    and three baselines (KD, OPD, and EOPD) across six mathematical
    reasoning benchmarks and three teacher--student distillation settings
    based on the Qwen3 model family. The best-performing result within each
    setup is highlighted in \textbf{bold}, and the second-best result is
    \underline{underlined}. Note that LSPD-RB results are reported with only \textit{10 training steps}, whereas all other methods require at least 40 training steps to converge stably.
    }
    \vspace{-1em}
    \label{tab:main_results}

    \renewcommand{\arraystretch}{1.10}

    \resizebox{0.99\textwidth}{!}{
    \begin{tabular}{c|c|cc|cc|cc|cc|cc|cc}
        \toprule
        \textbf{Model} & \textbf{Method}
        & \multicolumn{2}{c|}{\textbf{MATH-500}}
        & \multicolumn{2}{c|}{\textbf{Minerva}}
        & \multicolumn{2}{c|}{\textbf{Olympiad}}
        & \multicolumn{2}{c|}{\textbf{AMC23}}
        & \multicolumn{2}{c|}{\textbf{AIME24}}
        & \multicolumn{2}{c}{\textbf{AIME25}} \\

        &
        & Avg@16 & Pass@16
        & Avg@16 & Pass@16
        & Avg@16 & Pass@16
        & Avg@16 & Pass@16
        & Avg@16 & Pass@16
        & Avg@16 & Pass@16 \\
        \midrule

        % ============================================================
        % 8B -> 4B
        % ============================================================

        & KD
        & 79.73 & 94.00
        & \underline{38.51} & 58.09
        & 45.97 & 67.85
        & 51.20 & 77.11
        & 17.71 & 33.33
        & 17.08 & \underline{36.67} \\

        & OPD
        & 79.31 & \underline{95.20}
        & 35.73 & \underline{58.46}
        & 44.64 & 67.85
        & 48.49 & 79.97
        & 17.92 & 30.00
        & 15.83 & \textbf{40.00} \\

        & EOPD
        & 80.50 & 95.00
        & 37.78 & \textbf{59.56}
        & 46.36 & 68.00
        & 51.05 & 80.72
        & 17.92 & \underline{36.67}
        & 17.92 & \textbf{40.00} \\

        \rowcolor{blue!6}
        \cellcolor{white}
        & \textbf{LSPD}
        & \textbf{81.61} & \textbf{95.40}
        & \textbf{38.99} & \textbf{59.56}
        & \underline{47.22} & \underline{69.04}
        & \textbf{52.48} & \underline{81.93}
        & \underline{18.96} & \underline{36.67}
        & \underline{18.13} & \underline{36.67} \\

        \rowcolor{blue!6}
        \cellcolor{white}
        \multirow{-5}{*}{\makecell{8B\\$\downarrow$\\4B}}
        & \textbf{LSPD-RB}
        & \underline{81.55} & \underline{95.20}
        & 38.49 & 58.09
        & \textbf{47.41} & \textbf{69.33}
        & \underline{52.11} & \textbf{84.34}
        & \textbf{21.25} & \textbf{50.00}
        & \textbf{18.75} & \underline{36.67} \\

        \midrule

        % ============================================================
        % 4B -> 1.7B
        % ============================================================

        & KD
        & 66.75 & \underline{91.00}
        & 27.76 & 52.94
        & 30.52 & 58.52
        & 35.54 & \underline{66.27}
        & 8.13 & 26.67
        & 5.21 & 13.33 \\

        & OPD
        & 69.13 & 90.20
        & 26.93 & 53.68
        & 31.98 & 57.04
        & 34.64 & 65.06
        & 10.00 & 26.67
        & 6.88 & 16.67 \\

        & EOPD
        & 69.75 & 90.60
        & 27.92 & 54.04
        & 32.91 & 59.26
        & 35.84 & \underline{66.27}
        & 11.46 & \underline{33.33}
        & \underline{7.50} & \underline{20.00} \\

        \rowcolor{blue!6}
        \cellcolor{white}
        & \textbf{LSPD}
        & \textbf{70.86} & \textbf{91.80}
        & \underline{29.32} & \textbf{55.15}
        & \underline{33.47} & \underline{59.70}
        & \textbf{36.97} & \textbf{67.47}
        & \textbf{11.88} & 30.00
        & \textbf{8.33} & \textbf{23.33} \\

        \rowcolor{blue!6}
        \cellcolor{white}
        \multirow{-5}{*}{\makecell{4B\\$\downarrow$\\1.7B}}
        & \textbf{LSPD-RB}
        & \underline{70.69} & \underline{91.00}
        & \textbf{29.57} & \underline{54.41}
        & \textbf{33.50} & \textbf{61.04}
        & \underline{36.60} & \textbf{67.47}
        & \underline{11.67} & \textbf{36.67}
        & \textbf{8.33} & \textbf{23.33} \\

        \midrule

        % ============================================================
        % 1.7B -> 0.6B
        % ============================================================

        & KD
        & 49.61 & 79.60
        & 15.49 & \underline{41.18}
        & 18.98 & 41.63
        & 22.97 & 53.01
        & 2.08 & 13.33
        & \underline{1.88} & \underline{10.00} \\

        & OPD
        & 51.40 & 79.80
        & 11.28 & 37.50
        & \textbf{20.66} & 41.93
        & 22.89 & 51.81
        & \underline{4.17} & 13.33
        & 1.25 & \textbf{13.33} \\

        & EOPD
        & \underline{51.98} & 79.40
        & 15.85 & 38.60
        & 19.84 & 42.81
        & \underline{23.12} & 53.01
        & 3.33 & \underline{16.67}
        & 1.46 & \underline{10.00} \\

        \rowcolor{blue!6}
        \cellcolor{white}
        & \textbf{LSPD}
        & \textbf{52.04} & \textbf{81.60}
        & \textbf{17.14} & 40.81
        & \underline{20.39} & \underline{43.56}
        & \textbf{24.62} & \underline{54.22}
        & \textbf{5.00} & \textbf{20.00}
        & 1.46 & \underline{10.00} \\

        \rowcolor{blue!6}
        \cellcolor{white}
        \multirow{-5}{*}{\makecell{1.7B\\$\downarrow$\\0.6B}}
        & \textbf{LSPD-RB}
        & 51.56 & \underline{81.20}
        & \underline{16.36} & \textbf{41.54}
        & 20.12 & \textbf{44.59}
        & 22.97 & \textbf{65.06}
        & 3.96 & \underline{16.67}
        & \textbf{2.29} & \textbf{13.33} \\

        \bottomrule
    \end{tabular}
    }
\end{table*}

\begin{figure}[t]
    \centering
    \includegraphics[width=1.0\linewidth]{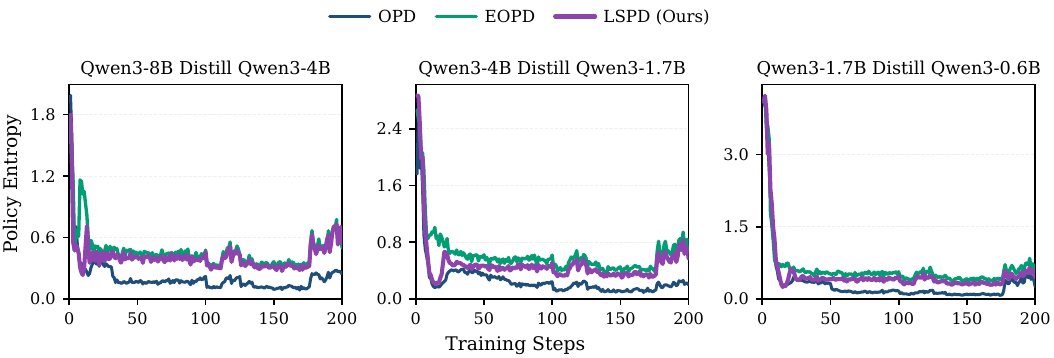}
    \vspace{-1em}
    \caption{\textbf{Student Entropy during Training.} We plot the student policy entropy over 200 training steps for OPD, EOPD, and LSPD across three teacher--student pairs. LSPD maintains an entropy level comparable to EOPD and consistently higher than standard OPD, indicating improved preservation of student policy diversity during distillation.}
    \label{fig:entropy}
\end{figure}

\vspace{.3em} \noindent \textbf{Results on Model Entropy.}
To examine the student's response diversity during distillation, we track its policy entropy throughout training across three teacher--student pairs in Figure~\ref{fig:entropy}. Both LSPD and EOPD \citep{jin2026entropy} maintain student entropy at approximately 0.5, consistently higher than standard OPD \citep{lu2025onpolicydistillation}. While EOPD augments the OPD objective with an entropy-gated forward KL term, LSPD encourages higher entropy directly through the entropy regularization term in Eq.~\ref{eqn:rpd-practical}, together with the robust log-probability matching objective. Despite these different formulations, both methods exhibit similar entropy behavior during training. This suggests that LSPD can maintain a higher level of student policy diversity than standard OPD with direct entropy-regularized maximization instead of relying on EOPD's entropy-gated forward KL formulation.

\vspace{.3em} \noindent \textbf{Results on Pass@$k$ Scaling.}
To further evaluate the multi-sample solution coverage and useful generation diversity of the distilled student model, we study how Pass@$k$ scales with the number of sampled responses. As reported in Table~\ref{tab:passk_scaling}, we distill a Qwen3-4B-Base student from a Qwen3-8B non-thinking teacher using LSPD and three baselines, KD, OPD, and EOPD, and evaluate Pass@$k$ over a broad range of sampling budgets, with $k\in\{1,2,4,8,16,32,64\}$. Remarkably, LSPD exhibits stronger scaling as the sampling budget increases. In particular, it achieves the best or tied-best Pass@32 and Pass@64 on all three benchmarks, reaching 85.54, 43.33, and 46.67 Pass@32, and 89.16, 46.67, and 50.00 Pass@64 on AMC23, AIME24, and AIME25, respectively. Averaged across the three benchmarks, LSPD achieves Pass@32/64 of 58.51/61.94, consistently outperforming the strongest competing method, EOPD (57.00/60.43), by +1.51 and +1.52 points, respectively. These results suggest that LSPD maintains competitive single-sample accuracy while providing stronger solution coverage under larger sampling budgets, indicating that its improvements are not limited to concentrating probability mass on a small set of high-probability responses.

\begin{table*}[t]
    \centering
    \caption{\textbf{Pass@$k$ Scaling Results.}
    We report Pass@$k$ results for
    $k\in\{1,2,4,8,16,32,64\}$ on Qwen3-4B-Base distilled from Qwen3-8B non thinking teacher. The best-performing result within each setup is highlighted in \textbf{bold}, and the second-best result is \underline{underlined}.}
    \vspace{-1em}
    \label{tab:passk_scaling}

    \renewcommand{\arraystretch}{1.10}

    \resizebox{0.99\textwidth}{!}{
    \begin{tabular}{c c|ccccccc}
        \toprule
        \textbf{Benchmark} & \textbf{Method}
        & \textbf{Pass@1}
        & \textbf{Pass@2}
        & \textbf{Pass@4}
        & \textbf{Pass@8}
        & \textbf{Pass@16}
        & \textbf{Pass@32}
        & \textbf{Pass@64} \\
        \midrule

        \multirow{4}{*}{\textbf{AMC23}}
        & KD
        & \underline{51.81}
        & 56.63
        & 66.27
        & 71.08
        & 77.11
        & \textbf{85.54}
        & \underline{87.95} \\

        & OPD
        & \underline{51.81}
        & 55.42
        & 63.86
        & 73.49
        & 79.97
        & 83.13
        & 86.75 \\

        & EOPD
        & \underline{51.81}
        & \underline{60.24}
        & \textbf{72.29}
        & \underline{75.90}
        & \underline{80.72}
        & \underline{84.34}
        & \underline{87.95} \\

        \rowcolor{blue!6}
        \cellcolor{white}
        & \textbf{LSPD}
        & \textbf{54.22}
        & \textbf{62.65}
        & \underline{71.08}
        & \textbf{79.52}
        & \textbf{81.93}
        & \textbf{85.54}
        & \textbf{89.16} \\
        \midrule

        \multirow{4}{*}{\textbf{AIME24}}
        & KD
        & \textbf{20.00}
        & \textbf{26.67}
        & \textbf{30.00}
        & \textbf{33.33}
        & \underline{33.33}
        & 36.67
        & 40.00 \\

        & OPD
        & \underline{16.67}
        & 20.00
        & 23.33
        & 26.67
        & 30.00
        & 36.67
        & 40.00 \\

        & EOPD
        & 13.33
        & \underline{23.33}
        & \textbf{30.00}
        & \textbf{33.33}
        & \textbf{36.67}
        & \underline{40.00}
        & \underline{43.33} \\

        \rowcolor{blue!6}
        \cellcolor{white}
        & \textbf{LSPD}
        & \underline{16.67}
        & \underline{23.33}
        & \underline{26.67}
        & \underline{30.00}
        & \textbf{36.67}
        & \textbf{43.33}
        & \textbf{46.67} \\
        \midrule

          \multirow{4}{*}{\textbf{AIME25}}
        & KD
        & 10.00
        & \underline{23.33}
        & \underline{26.67}
        & \underline{33.33}
        & \underline{36.67}
        & \underline{43.33}
        & 43.33 \\

        & OPD
        & \textbf{20.00}
        & \textbf{26.67}
        & \textbf{33.33}
        & \textbf{36.67}
        & \textbf{40.00}
        & \underline{43.33}
        & \underline{46.67} \\

        & EOPD
        & \underline{13.33}
        & 20.00
        & \textbf{33.33}
        & \textbf{36.67}
        & \textbf{40.00}
        & \textbf{46.67}
        & \textbf{50.00} \\

        \rowcolor{blue!6}
        \cellcolor{white}
        & \textbf{LSPD}
        & \underline{13.33}
        & 20.00
        & \underline{26.67}
        & \underline{33.33}
        & \underline{36.67}
        & \textbf{46.67}
        & \textbf{50.00} \\
        \bottomrule
    \end{tabular}
    }
    \vspace{-1em}
  \end{table*}

\subsection{Off-policy and LSPD-RB Replay Buffer Results}
\label{sec:off-policy}

Since our proposed objective naturally supports off-policy optimization, we evaluate its ability to leverage off-policy data under two settings. First, we vary the number of optimization steps per rollout batch as $N\in\{1,4,16,64\}$, resulting in increasingly off-policy updates while still training only on the most recently collected rollouts. Second, we consider a fully off-policy variant, LSPD-RB, which maintains a replay buffer containing all previously collected query--response pairs and performs updates using batches sampled from this buffer. Figure~\ref{fig:off-policy} reports the corresponding Avg@16 training curves on AMC23, AIME24, and AIME25 with a Qwen3-4B teacher model and a Qwen3-1.7B-Base student. For all settings in Figure~\ref{fig:off-policy}, including different choices of $N$ and the replay-buffer variant, each training step collects one rollout batch consisting of 64 prompts with 4 responses sampled per prompt.

For the semi-on-policy setting, increasing the number of updates per rollout batch consistently improves sample efficiency, requiring fewer rollout batches to reach saturated performance. This effect is particularly pronounced on AIME24 and AIME25: with $N\geq 4$, performance typically saturates after roughly 30 training steps, whereas $N=1$ requires more than 50 steps. The gains become substantially smaller beyond $N=4$, however, suggesting that the sample-efficiency benefit of additional updates eventually saturates. In contrast, the fully off-policy LSPD-RB variant reaches its saturated performance within approximately 10 training steps, substantially faster than all semi-on-policy variants. These results demonstrate that LSPD can effectively reuse historical data and that replay-based off-policy training can further improve its sample efficiency.

\begin{figure}[t]
    \centering
    \includegraphics[width=1.0\linewidth]{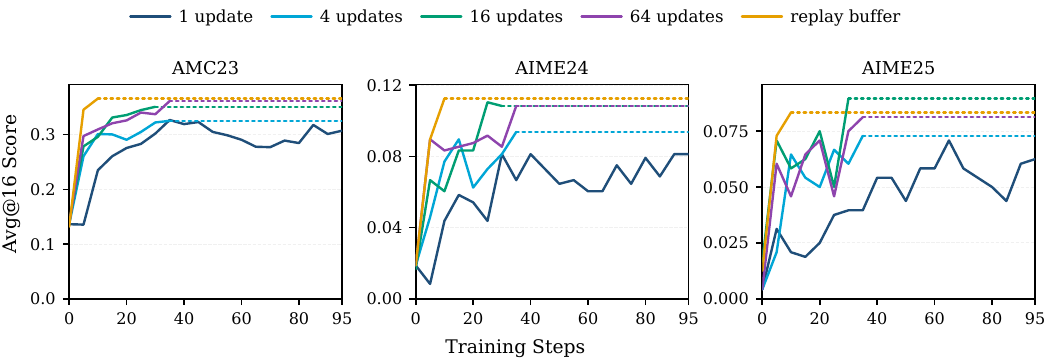}
    \vspace{-1em}
    \caption{\textbf{Off-Policy Training Curves.} We plot the evolution of Avg@16 on AMC23, AIME24, and AIME25 for LSPD with $N\in\{1,4,16,64\}$ optimization steps per rollout batch, together with the replay-buffer variant LSPD-RB. For LSPD-RB, after each newly collected rollout batch is added to the replay buffer, we perform 256 optimization steps on batches sampled from the buffer. We indicate the saturated performance after convergence using dotted horizontal lines. Each training step collects one rollout batch consisting of 64 prompts with 4 responses sampled per prompt.}
    \label{fig:off-policy}
    \vspace{-1em}
\end{figure}

\vspace{.3em}\noindent \textbf{Ablation Studies.} We ablate the entropy regularization component introduced in Eq.~\ref{eqn:rpd-practical} to study its effect on reasoning performance and solution diversity. As shown in Table~\ref{tab:entropy_reg}, removing entropy regularization has only a modest effect on Avg@16, but leads to a substantially larger degradation in Pass@16. Averaged across the six benchmarks, removing entropy regularization decreases Avg@16 by only $0.17$ points for LSPD and $0.75$ points for LSPD-RB, while Pass@16 drops by $2.26$ and $1.64$ points, respectively. Across both methods and all benchmarks, this corresponds to an average degradation of $0.46$ points in Avg@16 compared with $1.95$ points in Pass@16. Notably, Pass@16 decreases in all $12$ method--benchmark comparisons without entropy regularization, whereas the changes in Avg@16 are comparatively small and mixed. These results suggest that entropy regularization primarily improves the diversity and coverage of generated solutions, rather than substantially changing the average per-sample accuracy.

\begin{table}[t]
\centering
\caption{\textbf{Effect of Entropy Regularization.}
Avg@16 and Pass@16 performance of LSPD and LSPD-RB with and without
entropy regularization across six mathematical reasoning benchmarks.}
\vspace{-1em}
\label{tab:entropy_reg}

\renewcommand{\arraystretch}{1.15}
\setlength{\tabcolsep}{3.5pt}

\resizebox{\textwidth}{!}{%
\begin{tabular}{cc|cc|cc|cc|cc|cc|cc}
    \toprule
    \multirow{2}{*}{\textbf{Reg.}}
    & \multirow{2}{*}{\textbf{Method}}
    & \multicolumn{2}{c|}{\textbf{MATH-500}}
    & \multicolumn{2}{c|}{\textbf{Minerva}}
    & \multicolumn{2}{c|}{\textbf{Olympiad}}
    & \multicolumn{2}{c|}{\textbf{AMC23}}
    & \multicolumn{2}{c|}{\textbf{AIME24}}
    & \multicolumn{2}{c}{\textbf{AIME25}} \\

    &
    & \textbf{Avg@16} & \textbf{Pass@16}
    & \textbf{Avg@16} & \textbf{Pass@16}
    & \textbf{Avg@16} & \textbf{Pass@16}
    & \textbf{Avg@16} & \textbf{Pass@16}
    & \textbf{Avg@16} & \textbf{Pass@16}
    & \textbf{Avg@16} & \textbf{Pass@16} \\
    \midrule

    \multirow{2}{*}{\textbf{w/ Ent.}}
    & LSPD
    & 70.86 & 91.80
    & 29.32 & 55.15
    & 33.47 & 59.70
    & 36.97 & 67.47
    & 11.88 & 30.00
    & 8.33 & 23.33 \\

    & LSPD-RB
    & 70.69 & 91.00
    & 29.57 & 54.41
    & 33.50 & 61.04
    & 36.60 & 67.47
    & 11.67 & 36.67
    & 8.33 & 23.33 \\

    \midrule

    \multirow{2}{*}{\textbf{w/o Ent.}}
    & LSPD
    & 71.14 & 90.60
    & 28.65 & 51.84
    & 32.91 & 58.52
    & 37.12 & 66.27
    & 11.46 & 26.67
    & 8.54 & 20.00 \\

    & LSPD-RB
    & 69.39 & 90.60
    & 28.84 & 53.31
    & 33.65 & 60.59
    & 36.67 & 66.27
    & 10.00 & 33.33
    & 7.29 & 20.00 \\

    \bottomrule
\end{tabular}%
}

\end{table}

% \section{An Theoretical View on LSPD, OPD and SFT}
\section{Theoretical analysis}

In this section, we establish the theoretical advantage of the proposed algorithm by explicitly enforcing optimism and exploiting the value-based KL-regularized RL. The detailed proof is deferred to Appendix~\ref{sec:proof-theoretical-guarantee}. Following the notation in Section~\ref{sec:prelim}, we define the cumulative regret as
\begin{align}
\mathrm{Regret}(K) &:=
\textstyle{\sum_{k=1}^{K}}
\mathbb{E}_{\qb\sim\cD,\,\yb\sim\pi_k(\cdot|\qb)}
\left[
\sum_{t=1}^{T}
D_{\mathrm{KL}}\!\left(
\pi_k(\cdot|\xb_{<t})
\|
\pi^*(\cdot|\xb_{<t})
\right)
\right]
\nonumber\\
&=
\textstyle{\sum_{k=1}^{K}}
\mathbb{E}_{\qb\sim\cD}
\left[
D_{\mathrm{KL}}\!\left(
\pi_k(\cdot|\qb)
\,\|\,
\pi^*(\cdot|\qb)
\right)
\right].
\label{eq:regret}
\end{align}
Here $\xb_{<t}=(\qb,\yb_{<t})$, $K$ counts rollout/update rounds, and $T$ is the response-length and the second equality is due to the chain rule of KL divergence. The policy $\pi^*$ denotes the optimal policy. Similarly with~\citet{sriraman2026behavior}, we assume the noisy teacher feedback in $\pi^E$ as detailed in Assumption~\ref{ass:noise}. Then the following assumption is also crucial in understanding OPD:

\begin{assumption}
\label{ass:realizability}
Let
$ \mathcal{R} \subseteq \left\{
r: \mathcal{X}\times\mathcal{Y}
\rightarrow [-B,B] \right\}
$
be a reward class with $B < \infty$. For the optimal policy $\pi^*$ and the reference policy $\pi^{\mathrm{ref}}$, we assume  $r^*(\xb_{<t},y_t) := \log  \frac{ \pi^*(y_t|\xb_{<t}) }{\pi^{\mathrm{ref}}(y_t|\xb_{<t}) } \in \cR$.
\end{assumption}

In practice, one can choose $\pi^{\mathrm{ref}}$ as the initial student model $\pi_0$ thus Assumption~\ref{ass:realizability} assumes the well coverage between the student and teacher model.

\begin{theorem}[Logarithmic reverse-KL regret, informal]
\label{thm:main}
Under Assumptions~\ref{ass:noise} and~\ref{ass:realizability}, with high probability, Algorithm~\ref{alg:rpd-theory} satisfies $\mathrm{Regret}(K) = \mathcal{O} (d_E \log(N_{\mathcal R}K))$ where $d_E, N_{\cR}$ are both complexity measurements in the order of $\tilde \cO(d)$ when the reward is a $d$-dimensional linear function. 
\end{theorem}

We defer the detailed proof and further discussion to Appendix~\ref{sec:proof-theoretical-guarantee}, and highlight how Theorem~\ref{thm:main} informs the comparison between reverse-KL-based OPD and forward-KL-based SFT.

\begin{remark}
Theorem~\ref{thm:main} implies a $\widetilde{\cO}(\epsilon^{-1})$ sample complexity for obtaining a mixture policy $\bar\pi = \frac{1}{K}\sum_{k=1}^{K}\pi_k$ whose reverse-KL divergence to $\pi^*$ is at most $\epsilon$ under \textbf{bounded log-ratio reward}. By comparison, \citet{foster2024behavior} establish a $\widetilde{\cO}(\epsilon^{-1})$ sample complexity for attaining an $\epsilon$-small imitation performance gap through maximum likelihood estimation (MLE) with a \textbf{deterministic expert}. Although these guarantees concern different error criteria, they highlight complementary conditions supporting fast learning. 
In reverse KL minimization including OPD and our LSPD, the fast rates depends on the coverage between the student and teacher model with $B<\infty$ according to Assumption~\ref{ass:realizability}. This dependence on the coverage is also reported empirically in \cite{li2026rethinking,xing2026trust,fu2026rethinking}. In contrast, the SFT with MLE exploits expert determinism and exhibits the ``mass-covering'' tendency in minimizing the forward KL. This comparison explains the practical implementations for SFT in establishing an initial policy  followed by reverse-KL-based refinement with adequate coverage.
\end{remark}

% \begin{theorem}[NAILGUN with unknown corruption, \citet{sriraman2026behavior}]
% \label{thm:nail_unknown_corruption}
% Consider a finite policy class $\Pi$ containing a deterministic
% expert $\pi^*$, with noisy feedback from
% $\pi^E=(1-\zeta)\pi^*+\zeta\nu$.
% Assume $\nu_t(y_t|\xb_{<t})\le\rho$ for every $(t,\xb_{<t},y_t)$ and
% $\zeta\le\alpha$, where $\alpha(1+\rho)<1$.
% Without knowing $\zeta$ or $\nu$, NAILGUN returns a policy
% $\widehat\pi$ after $K$ online interaction rounds satisfying
% \begin{equation}
% \mathbb E\!\left[
% D_{\mathrm H^2}\!\left(
% P^{\pi^*},P^{\widehat\pi}
% \right)
% \right]
% \lesssim
% \frac{\log|\Pi|}
% {K\bigl(1-\alpha(1+\rho)\bigr)^2}.
% \end{equation}
% {\color{red} TODO: need further alignment.}
% \end{theorem}

\section{Conclusion}
We presented Least Square Policy Distillation (LSPD), an RL-inspired framework that connects reverse-KL distillation to KL-regularized policy optimization through a teacher-induced reward. This perspective motivates an optimistic least-squares formulation that enables trajectory reuse with a purely off-policy replay buffer. Experiments across six mathematical reasoning benchmarks and multiple teacher--student settings show improved reasoning performance, stronger Pass@$k$ scaling, and greater rollout efficiency. In particular, LSPD-RB achieves performance comparable to vanilla OPD using only the first $25\%$ of rollout batches. Together with a sharp reverse-KL regret guarantee, these results show how value-based RL principles can improve policy distillation by preserving policy diversity and enhancing sample efficiency through off-policy learning.

\bibliography{iclr2027_conference}
\bibliographystyle{iclr2027_conference}

\clearpage
\appendix
\section{Hyperparameters and Implementation Details}

\subsection{Details on Huber-Type Penalty in Practical Implementation}
\label{sec:huber}
For improved numerical stability, we employ a Huber-type penalty in the practical implementations of LSPD and LSPD-RB. Specifically, let $\Delta_t(\pi,\pi^E)=\log \pi(y_t|\xb_{<t})-\log \pi^E(y_t|\xb_{<t})$, we define
\begin{align}
\label{eq:huber}
\psi\!\left(\Delta_t(\pi,\pi^E)\right)
=
\begin{cases}
\Delta_t^2(\pi,\pi^E),
& \left|\Delta_t(\pi,\pi^E)\right| \leq c, \\[4pt]
2c\left|\Delta_t(\pi,\pi^E)\right| - c^2,
& \left|\Delta_t(\pi,\pi^E)\right| > c,
\end{cases}
\end{align}
where $c>0$ denotes the threshold separating the quadratic and linear regimes. This formulation preserves the original quadratic objective when the log-probability discrepancy is moderate, while replacing the quadratic growth with a linear penalty for large discrepancies, thereby reducing the influence of extreme outliers and improving optimization stability. We provide an ablation study of this design choice in Appendix~\ref{sec:addition-exp}.

\subsection{Training Details}

Our experimental implementation of LSPD and LSPD-RB builds upon the on-policy distillation codebase of \citet{li2026rethinking} and adopts the EOS corrections in \citet{yang2026eos}. Detailed training hyperparameters are summarized in Table~\ref{tab:hyperparameters}. For comparison, we reproduce the OPD~\citep{lu2025onpolicydistillation}, EOPD~\citep{jin2026entropy}, and KD~\citep{hinton2015distilling,kim2016sequence} baselines within the same implementation framework and under the same experimental setup. Across all methods, we keep the training batch size, rollout and evaluation configurations, and optimizer settings fixed to ensure a controlled comparison. All experiments are conducted using four NVIDIA RTX PRO 6000 GPUs, each equipped with 96GB of VRAM.

\begin{table*}[t]
      \centering
      \small
      \setlength{\tabcolsep}{10pt}
      \renewcommand{\arraystretch}{1.08}
      \caption{\textbf{Hyperparameters for LSPD and LSPD-RB.} We report the training hyperparameters used for our proposed LSPD and LSPD-RB methods across all experiments. For the off-policy experiments without a replay buffer, we vary the number of optimizer updates per rollout iteration over $N\in\{1,4,16,64\}$.}

      \label{tab:hyperparameters}
      \begin{tabular}{@{}lcc@{}}
          \toprule
          \textbf{Hyperparameter}
          & \textbf{w/o Replay Buffer}
          & \textbf{w/ Replay Buffer} \\
          \midrule
          Entropy regularization coefficient $\mu^{-1}$
              & $0.1$ & $0.1$ \\
          Squared-loss robustification
              & Huber & Huber \\
          Huber transition threshold $c$
              & $5.0$ & $5.0$ \\
          \midrule
          Optimizer
              & AdamW & AdamW \\
          Learning rate
              & $10^{-6}$ & $10^{-6}$ \\
          Learning-rate schedule
              & Constant & Constant \\
          Warmup steps
              & $0$ & $0$ \\
          Optimizer momentum coefficients
              & $(0.9,\,0.999)$ & $(0.9,\,0.999)$ \\
          Weight decay
              & $0.01$ & $0.01$ \\
          Maximum gradient norm
              & $1.0$ & $1.0$ \\
          \midrule
          Maximum prompt length (tokens)
              & $1{,}024$ & $1{,}024$ \\
          Maximum response length (tokens)
              & $7{,}168$ & $7{,}168$ \\
          Rollout sampling temperature
              & $1.0$ & $1.0$ \\
          Teacher distribution temperature
              & $1.0$ & $1.0$ \\
          Vocabulary truncation
              & None & None \\
          Prompts per rollout iteration
              & $64$ & $64$ \\
          Responses per prompt
              & $4$ & $4$ \\
          New trajectories per rollout iteration
              & $256$ & $256$ \\
          Trajectories per optimizer update
              & $64$ & $64$ \\
          Optimizer updates per rollout iteration
              & $4$ & $256$ \\
          Rollout iterations
              & $100$ & $100$ \\
          \midrule
          Replay buffer capacity (trajectories)
              & --- & $65{,}536$ \\
          Replay eviction policy
              & --- & First in, first out \\
          Replay sampling
              & --- & Uniform \\
          \bottomrule
      \end{tabular}
  \end{table*}
  
\subsection{Prompt Template}
\label{sec:prompt-template}

For both training rollouts and evaluation, we follow the mathematical reasoning prompt template used in Test-Time Reinforcement Learning (TTRL) \citep{zuo2026ttrl}. Specifically, we append the following instruction to each problem:

\begin{promptbox}{Prompt Template}
\texttt{Please reason step by step, and put your final answer within
\textbackslash boxed\{\}.}
\end{promptbox}

\section{Additional Experimental Results}
\label{sec:addition-exp}
\paragraph{Out-of-Domain Results.}
We further evaluate whether improvements from LSPD transfer beyond the mathematical reasoning domain used for distillation. Specifically, although the student is trained exclusively on mathematical dataset DAPO-Math-17K \citep{yu2026dapo}, we evaluate KD~\citep{hinton2015distilling,kim2016sequence}, OPD~\citep{lu2025onpolicydistillation}, EOPD~\citep{jin2026entropy}, and our proposed LSPD and replay-buffer variant LSPD-RB on four out-of-domain benchmarks spanning general knowledge, commonsense reasoning, and code generation: GPQA~\citep{rein2023gpqa}, MMLU~\citep{hendrycks2020measuring}, WinoGrande~\citep{sakaguchi2021winogrande}, and HumanEval~\citep{chen2021evaluating}. We report 0-shot accuracy on GPQA Main (448 questions), 5-shot accuracy on MMLU (14,042 questions), and 0-shot validation accuracy on WinoGrande (1,267 questions). For HumanEval, we report accuracy using one greedy completion for each of the 164 tasks. For all out-of-domain experiments, we leverage a Qwen3-4B non-thinking teacher and a Qwen3-1.7B-Base student.

Results are shown in Table~\ref{tab:general_benchmarks}. Averaged across the four out-of-domain benchmarks, LSPD achieves a score of $56.02$, improving upon the strongest baseline, EOPD ($55.01$), by $+1.02$ percentage points. LSPD-RB further increases the average score to $56.84$, corresponding to a $+1.84$ point improvement over EOPD. Notably, LSPD-RB achieves the best performance on GPQA, MMLU, and HumanEval, while LSPD obtains the best result on WinoGrande. These results indicate that the gains from LSPD are not restricted to the in-domain mathematical reasoning tasks used during training; instead, the learned policy retains and, in several cases, improves general reasoning and coding capabilities, suggesting better preservation of the student's out-of-domain capabilities during distillation.

\begin{table}[t]
    \centering
    \caption{\textbf{Results on Out-of-Domain Benchmarks.}
    Performance comparison of KD, OPD, EOPD, LSPD, and LSPD-RB on
    GPQA, MMLU, WinoGrande, and HumanEval. The student is trained only on
    mathematical reasoning data. The best result on each benchmark is shown
    in \textbf{bold}, and the second-best result is \underline{underlined}.}
    \label{tab:general_benchmarks}

    \renewcommand{\arraystretch}{1.15}
    \setlength{\tabcolsep}{8pt}

    \begin{tabular}{lcccc}
        \toprule
        \textbf{Method}
        & \textbf{GPQA}
        & \textbf{MMLU}
        & \textbf{WinoGrande}
        & \textbf{HumanEval} \\
        \midrule
        KD      & \underline{30.13} & 61.24 & 63.85 & 64.63 \\
        OPD     & 27.68 & 61.14 & 63.61 & 67.07 \\
        EOPD    & 27.90 & 61.47 & \underline{64.80} & 65.85 \\
        \rowcolor{blue!6}
        LSPD    & 29.91 & \underline{61.61} & \textbf{64.88} & \underline{67.68} \\
        \rowcolor{blue!6}
        LSPD-RB & \textbf{30.58} & \textbf{62.03} & 64.64 & \textbf{70.12} \\
        \bottomrule
    \end{tabular}
\end{table}

\paragraph{Ablation on Huber-Type Robust Penalty.}
We ablate the Huber-type robust penalty introduced in Eq.~\ref{eq:huber}. Specifically, we distill a Qwen3-1.7B-Base student from a Qwen3-4B non-thinking teacher using LSPD-RB, with and without the Huber-type penalty. The results are shown in Figure~\ref{fig:huber-ablation}. Across AMC23, AIME24, and AIME25, the Avg@16 results show that incorporating the Huber-type penalty consistently improves performance, yielding an average gain of $+0.50$ points.

We further analyze the distribution of token-level log-probability differences using the step-10 checkpoint of the distilled Qwen3-1.7B-Base model. We sample 128 responses from 32 training prompts in DAPO-Math-17K, resulting in approximately $335$K evaluated tokens. We find that $99.655\%$ of tokens fall within the quadratic region of Eq.~\ref{eq:huber}, whereas only $0.345\%$ exhibit sufficiently large log-probability discrepancies to enter the linear region. This suggests that the objective behaves predominantly as a quadratic penalty during training, while the linear tail robustly handles rare large discrepancies and improves training stability.

\begin{figure*}[t]
  \centering
  \begin{minipage}[b]{0.43\textwidth}
      \centering
      \includegraphics[width=\linewidth]
          {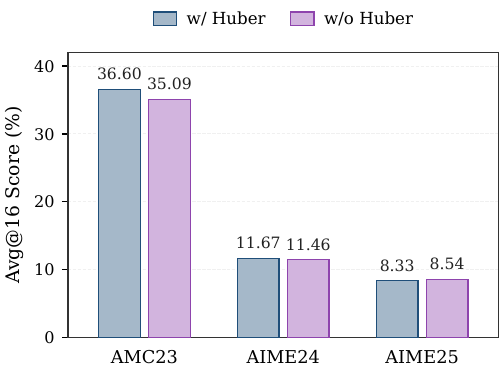}
      \par\smallskip
      {\small (a) Performance with and without Huber.\par}
  \end{minipage}\hfill
  \begin{minipage}[b]{0.55\textwidth}
      \centering
      \includegraphics[width=\linewidth]
          {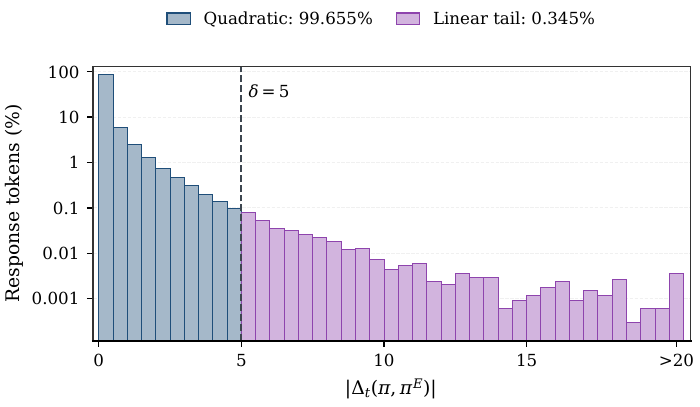}
      \par\smallskip
      {\small (b) Log-probability difference magnitudes.\par}
  \end{minipage}

  \caption{
     \textbf{Huber-Type Penalty Ablation.}
    We evaluate the effect of the Huber-type penalty on AMC23, AIME24, and AIME25. The left figure compares Avg@16 performance with and without the Huber-type penalty, while the right figure shows the distribution of the log-ratio discrepancy $\Delta_t(\pi,\pi^E)$. Overall, incorporating the Huber-type penalty yields a small but consistent improvement in performance. We further observe that the objective operates almost entirely in its quadratic regime: only $0.345\%$ of the evaluated token-level discrepancies satisfy $\Delta_t(\pi,\pi^E)>5$ and therefore fall into the linear tail. This suggests that the Huber formulation primarily behaves as a quadratic penalty in practice, while its linear tail provides additional robustness to the small fraction of large log-ratio discrepancies.
  }
  \label{fig:huber-ablation}
\end{figure*}

\begin{algorithm}[t]
\caption{Least Square Policy Distillation (Theoretical)}
\label{alg:rpd-theory}
\begin{algorithmic}[1]
\REQUIRE Reference policy $\pi^{\mathrm{ref}}$, reward class $\mathcal R$, teacher feedback $\pi_k^E$, rounds $K$, response-length bound $T$, parameters $\lambda,\beta$; $\eta=1$.
\STATE Initialize $\mathcal D_0=\emptyset$, $\mathcal C_0=\mathcal R$, and $\pi_1=\pi_{r_0^+}$ with $r_0^+=\sup_{r\in\mathcal R}r$ pointwise.
\FOR{$k=1,\ldots,K$}
    \STATE Sample $\qb_k\sim\mathcal D$ and initialize $\xb_{k,<1}=(\qb_k,\emptyset)$.
    \FOR{$t=1,\ldots,T$}
        \STATE Sample $y_{k,t}\sim\pi_k(\cdot|\xb_{k,<t})$.
        \STATE Observe $\ell_{k,t}=\log\!\left(\pi_k^E(y_{k,t}|\xb_{k,<t})/\pi^{\mathrm{ref}}(y_{k,t}|\xb_{k,<t})\right)$.
        \STATE Append the token: $\xb_{k,<t+1}=(\xb_{k,<t},y_{k,t})$; stop at termination.
    \ENDFOR
    \STATE Store the prefix--token pairs and targets from rollout $k$ in $\mathcal D_k$.
    \STATE Fit $\widehat r_k$ to all stored token targets by Eq.~\ref{eq:app-least-squares}.
    \STATE Construct the estimator-centered confidence set $\mathcal C_k$ in Eq.~\ref{eq:app-confidence-set}.
    \STATE Set $r_k^+(\xb_{<t},y_t)=\sup_{r\in\mathcal C_k}r(\xb_{<t},y_t)$ for every prefix--token pair.
    \STATE Set $\pi_{k+1}(y_t|\xb_{<t})\propto\pi^{\mathrm{ref}}(y_t|\xb_{<t})\exp\!\left(r_k^+(\xb_{<t},y_t)\right)$.
\ENDFOR
\end{algorithmic}
\end{algorithm}

\section{Theoretical Analysis and Proofs}
\label{sec:proof-theoretical-guarantee}

This appendix proves Theorem~\ref{thm:main} by adapting the contextual-bandit analysis of \citet{zhao2025logarithmic} to conditional next-token distributions. We apply the sharp bandit inequality at each fixed prefix and then sum it along student-generated responses using the KL chain rule. No transition model, action-value function, or Bellman recursion is required. The reduction uses the special reward $r^*=\log(\pi^*/\pi^{\mathrm{ref}})$, whose target partition function equals one at every prefix; deterministic token concatenation alone would not justify a bandit analysis for an arbitrary sequential reward.

We first assume $N_{\mathcal R}=|\mathcal R|<\infty$ and give the covering-number extension in Appendix~\ref{sec:app-covering}. Throughout, $k$ indexes rollout/update rounds and $t$ indexes tokens in the finite vocabulary $\mathcal Y$. Queries are drawn from the dataset distribution $\mathcal D$ of Section~\ref{sec:prelim}. The policy $\pi_k$ remains fixed while generating $\yb_k$, and all its token observations become available for the update to $\pi_{k+1}$. Responses have at most $T$ tokens. Equivalently, completed responses may be padded with an absorbing token on which all policies agree and all rewards and regression errors are zero.

\subsection{Setup and Reverse-KL Identity}

Let $\mathcal F_{k-1}$ be the history before sampling query $\qb_k$. Within rollout $k$, let $\mathcal F_{k,t-1}$ contain this history, the current query, and the tokens and teacher feedback preceding position $t$. Thus $\xb_{k,<t}=(\qb_k,\yb_{k,<t})$ is known before drawing $y_{k,t}$. We use $\mathcal F_k$ for the history after the complete rollout and its feedback.

\begin{assumption}[Conditionally unbiased teacher]
\label{ass:noise}
There exists a target policy $\pi^*$ such that, for every observed prefix--token pair $(\xb_{k,<t},y_{k,t})$,
\begin{align*}
    \log\pi_k^E(y_{k,t}|\xb_{k,<t})
    &=\log\pi^*(y_{k,t}|\xb_{k,<t})+\epsilon_{k,t},\\
    \mathbb E[\epsilon_{k,t}|\mathcal F_{k,t-1},y_{k,t}]&=0.
\end{align*}
Moreover, the noise is conditionally $\sigma$-sub-Gaussian: for every $u\in\mathbb R$,
\[
    \mathbb E\!\left[
        \exp(u\epsilon_{k,t})
        \,\middle|\,
        \mathcal F_{k,t-1},y_{k,t}
    \right]
    \leq\exp(\sigma^2u^2/2).
\]
\end{assumption}

\begin{remark}
Explicit modeling of teacher noise has precedent in the theoretical study of on-policy distillation. In particular, \citet{sriraman2026behavior} distinguish a clean expert $\pi^*$ from an observed teacher corrupted at the token level according to
$\pi^E(y_t|\xb_{<t})=(1-\zeta)\pi^*(y_t|\xb_{<t})+\zeta\nu_t(y_t|\xb_{<t})$,
where $\zeta\in[0,1)$ is the corruption level and $\nu_t$ is a corruption policy. Their analysis considers known corruption and, for unknown corruption, imposes a deterministic clean expert together with smoothness and margin conditions on the corruption. Assumption~\ref{ass:noise} follows the same motivation of distinguishing target behavior from imperfect teacher feedback, but adopts a different statistical model: the observed token-level feedback $\log\pi_k^E(y_{k,t}|\xb_{k,<t})$ is conditionally unbiased for $\log\pi^*(y_{k,t}|\xb_{k,<t})$, with sub-Gaussian deviations. Conditional unbiasedness centers the induced reward-regression target at $r^*$, while the tail condition enables concentration under adaptive prefix and token sampling. Thus, introducing explicit structural assumptions on teacher noise is consistent with prior theoretical work. Our particular condition is a complementary analytical idealization tailored to least-squares reward estimation, rather than a consequence of their mixture-corruption model.
\end{remark}

\begin{definition}[Reward representation and regularized objective]
\label{def:app-reward-objective}
For a fixed reference policy $\pi^{\mathrm{ref}}$, define the target token reward by
\begin{equation}
    r^*(\xb_{<t},y_t)
    :=
    \log\frac{\pi^*(y_t|\xb_{<t})}{\pi^{\mathrm{ref}}(y_t|\xb_{<t})}.
    \label{eq:app-target-reward}
\end{equation}
For any reward function $r$, define its prefix-wise partition function and conditional Gibbs policy by
\begin{align}
    Z_r(\xb_{<t})
    &:=
    \sum_{y_t\in\mathcal Y}
    \pi^{\mathrm{ref}}(y_t|\xb_{<t})
    \exp\!\bigl(r(\xb_{<t},y_t)\bigr),
    \label{eq:app-partition}\\
    \pi_r(y_t|\xb_{<t})
    &:=
    \frac{\pi^{\mathrm{ref}}(y_t|\xb_{<t})
    \exp\!\bigl(r(\xb_{<t},y_t)\bigr)}{Z_r(\xb_{<t})}.
    \label{eq:app-gibbs-policy}
\end{align}
The sequence distribution is the autoregressive product of these conditionals, not a single sequence-level Gibbs normalization. With $\eta=1$, the regularized objective is
\begin{equation}
    J(\pi)
    :=
    \mathbb E_{\qb\sim\mathcal D,\,\yb\sim\pi(\cdot|\qb)}
    \left[
        \sum_{t=1}^{T}
        \left(
            r^*(\xb_{<t},y_t)
            -
            \log\frac{\pi(y_t|\xb_{<t})}{\pi^{\mathrm{ref}}(y_t|\xb_{<t})}
        \right)
    \right].
    \label{eq:app-regularized-objective}
\end{equation}
\end{definition}

\paragraph{Reverse-KL identity.}
Substituting Eq.~\ref{eq:app-target-reward} into Eq.~\ref{eq:app-regularized-objective} and conditioning on each prefix gives
\begin{align}
    J(\pi)
    &=-\mathbb E_{\qb\sim\mathcal D,\,\yb\sim\pi(\cdot|\qb)}
    \left[
        \sum_{t=1}^{T}
        D_{\mathrm{KL}}\!\left(
            \pi(\cdot|\xb_{<t})\,\|\,\pi^*(\cdot|\xb_{<t})
        \right)
    \right]
    \nonumber\\
    &=-\mathbb E_{\qb\sim\mathcal D}
    \left[
        D_{\mathrm{KL}}\!\left(
            \pi(\cdot|\qb)\,\|\,\pi^*(\cdot|\qb)
        \right)
    \right],
    \label{eq:app-objective-kl-identity}
\end{align}
where the second equality is the autoregressive KL chain rule. In particular, $Z_{r^*}(\xb_{<t})=1$ and $\pi_{r^*}=\pi^*$ at every prefix. Hence $J(\pi^*)=0$, and
\begin{equation}
    J(\pi^*)-J(\pi)
    =
    \mathbb E_{\qb\sim\mathcal D,\,\yb\sim\pi(\cdot|\qb)}
    \left[
        \sum_{t=1}^{T}
        D_{\mathrm{KL}}\!\left(
            \pi(\cdot|\xb_{<t})\,\|\,\pi^*(\cdot|\xb_{<t})
        \right)
    \right].
    \label{eq:app-gap-kl-identity}
\end{equation}
Thus $\mathrm{Regret}(K)=\sum_{k=1}^{K}[J(\pi^*)-J(\pi_k)]$ is exactly the regret in Eq.~\ref{eq:regret}, not a token-averaged surrogate.

\begin{definition}[Regression target and reward estimator]
\label{def:app-regression}
At token $t$ of rollout $k$, the observed regression target is
\begin{equation}
    \ell_{k,t}
    :=
    \log\frac{\pi_k^E(y_{k,t}|\xb_{k,<t})}
    {\pi^{\mathrm{ref}}(y_{k,t}|\xb_{k,<t})}.
    \label{eq:app-observed-target}
\end{equation}
Let $\mathcal D_k$ store the prefix--token pairs
$\{(\xb_{i,<u},y_{i,u}):1\leq i\leq k,\,1\leq u\leq T\}$
and their observed targets. The empirical version of Eq.~\ref{eqn:regression} is
\begin{equation}
    \widehat r_k
    \in\arg\min_{r\in\mathcal R}
    \sum_{i=1}^{k}\sum_{u=1}^{T}
    \bigl(r(\xb_{i,<u},y_{i,u})-\ell_{i,u}\bigr)^2.
    \label{eq:app-least-squares}
\end{equation}
We use $\mathcal D_0=\emptyset$ and any $\widehat r_0\in\mathcal R$.
\end{definition}

Under Assumption~\ref{ass:noise},
\begin{equation}
    \ell_{k,t}=r^*(\xb_{k,<t},y_{k,t})+\epsilon_{k,t},
    \qquad
    \mathbb E[\epsilon_{k,t}|\mathcal F_{k,t-1},y_{k,t}]=0.
    \label{eq:app-regression-model}
\end{equation}
The conditional formulation permits adaptive prefixes and does not require independent observations within a response.

\begin{definition}[Cumulative eluder complexity]
\label{def:app-eluder-complexity}
Fix $\lambda>0$. After $k$ completed rollouts, define
\begin{equation}
    U_k(\xb_{<t},y_t)
    :=
    \sup_{r_1,r_2\in\mathcal R}
    \frac{
        |r_1(\xb_{<t},y_t)-r_2(\xb_{<t},y_t)|
    }{
        \sqrt{
            \lambda+
            \sum_{i=1}^{k}\sum_{u=1}^{T}
            \bigl(r_1(\xb_{i,<u},y_{i,u})-r_2(\xb_{i,<u},y_{i,u})\bigr)^2
        }
    },
    \label{eq:app-uncertainty}
\end{equation}
including $k=0$ with an empty denominator sum. For $K$ rollout rounds, define
\begin{equation}
    d_E(\mathcal R,\lambda,K,T)
    :=
    \sup_{(\qb_k,\yb_k)_{k=1}^{K}}
    \sum_{k=1}^{K}\sum_{t=1}^{T}
    \min\!\left\{
        1,\,U_{k-1}(\xb_{k,<t},y_{k,t})^2
    \right\},
    \label{eq:app-eluder-complexity}
\end{equation}
where the supremum ranges over valid query--response sequences, $\xb_{k,<t}=(\qb_k,\yb_{k,<t})$, and $U_{k-1}$ uses only their first $k-1$ completed rollouts. We abbreviate this quantity as $d_E$.
\end{definition}

This is the rollout-batched token-level version of the cumulative contextual-bandit complexity. It bounds the clipped squared widths along every realized adaptive sequence. Importantly, its denominator does not use tokens from the current rollout before the policy is updated. At $T=1$, it is the usual bandit definition. Retaining only historical observations at the same token position can only enlarge each width; consequently, $d_E$ is at most the sum of the $T$ corresponding $K$-round bandit complexities. Thus response-length dependence is explicit in the complexity definition, without introducing a transition-estimation or Bellman-error term.

\begin{definition}[Confidence set, bonus, and optimistic reward]
\label{def:app-optimistic-reward}
For $\delta\in(0,1)$, set
\begin{equation}
    \beta
    :=
    \max\!\left\{
        2B,\,
        \sqrt{16\sigma^2\log\frac{2N_{\mathcal R}K}{\delta}+2\lambda}
    \right\}.
    \label{eq:app-beta}
\end{equation}
Define the estimator-centered confidence set by
\begin{equation}
    \mathcal C_k
    :=
    \left\{
        r\in\mathcal R:
        \lambda+
        \sum_{i=1}^{k}\sum_{u=1}^{T}
        \bigl(r(\xb_{i,<u},y_{i,u})-\widehat r_k(\xb_{i,<u},y_{i,u})\bigr)^2
        \leq\beta^2
    \right\}.
    \label{eq:app-confidence-set}
\end{equation}
Its pointwise optimistic reward and an auxiliary confidence bonus are
\begin{equation}
    r_k^+(\xb_{<t},y_t)
    :=\sup_{r\in\mathcal C_k}r(\xb_{<t},y_t),
    \qquad
    b_k(\xb_{<t},y_t)
    :=\min\{2B,\,\beta U_k(\xb_{<t},y_t)\}.
    \label{eq:app-bonus}
\end{equation}
\end{definition}

For noisy feedback, Eq.~\ref{eq:app-confidence-set} makes the theoretical confidence construction precise by centering at $\widehat r_k$: an uncentered squared residual against noisy targets includes the accumulated noise energy and is not, by itself, a logarithmic-radius confidence set. Algorithm~\ref{alg:rpd-theory} and the proof both use the same pointwise optimistic reward $r_k^+$; $b_k$ is only a bound on its estimation error. Since $\widehat r_k\in\mathcal C_k$, the set is always nonempty. The policy update is
\begin{equation}
    \pi_{k+1}(y_t|\xb_{<t})
    =\pi_{r_k^+}(y_t|\xb_{<t}).
    \label{eq:app-policy-update}
\end{equation}
In particular, $\mathcal C_0=\mathcal R$ and $\pi_1=\pi_{r_0^+}$. These are conditional token updates; they do not assert that a prefix-wise Gibbs policy globally maximizes an arbitrary sequence-reward objective.

\subsection{Objective Decomposition and the Sharp One-Step Bound}

For a fixed prefix $\xb_{<t}$, define
\begin{equation}
    \Delta(\xb_{<t},r)
    :=
    -\log Z_r(\xb_{<t})
    +
    \mathbb E_{y_t\sim\pi_r(\cdot|\xb_{<t})}
    \left[r(\xb_{<t},y_t)-r^*(\xb_{<t},y_t)\right].
    \label{eq:app-functional-gap}
\end{equation}

\begin{lemma}[Objective decomposition]
\label{lem:app-objective-decomposition}
For every $k\geq1$,
\begin{equation}
    J(\pi^*)-J(\pi_k)
    =
    \mathbb E_{\qb\sim\mathcal D,\,\yb\sim\pi_k(\cdot|\qb)}
    \left[
        \sum_{t=1}^{T}
        \bigl(
            \Delta(\xb_{<t},r_{k-1}^+)-\Delta(\xb_{<t},r^*)
        \bigr)
    \right].
    \label{eq:app-objective-decomposition}
\end{equation}
\end{lemma}

\begin{proof}
For a fixed prefix, Eq.~\ref{eq:app-gibbs-policy} gives
\begin{equation}
    \log\frac{\pi_k(y_t|\xb_{<t})}{\pi^{\mathrm{ref}}(y_t|\xb_{<t})}
    =r_{k-1}^+(\xb_{<t},y_t)-\log Z_{r_{k-1}^+}(\xb_{<t}).
    \label{eq:app-gibbs-log-ratio}
\end{equation}
Using $r^*=\log(\pi^*/\pi^{\mathrm{ref}})$, we obtain the prefix-wise identity
\begin{align}
    D_{\mathrm{KL}}\!\left(
        \pi_k(\cdot|\xb_{<t})\,\|\,\pi^*(\cdot|\xb_{<t})
    \right)
    &=
    -\log Z_{r_{k-1}^+}(\xb_{<t})
    \nonumber\\
    &\quad+
    \mathbb E_{y_t\sim\pi_k(\cdot|\xb_{<t})}
    \left[r_{k-1}^+(\xb_{<t},y_t)-r^*(\xb_{<t},y_t)\right].
    \label{eq:app-context-gap}
\end{align}
Because $Z_{r^*}(\xb_{<t})=1$ and $\pi_{r^*}=\pi^*$, the right-hand side is $\Delta(\xb_{<t},r_{k-1}^+)-\Delta(\xb_{<t},r^*)$, with $\Delta(\xb_{<t},r^*)=0$. Sum over tokens and average over prefixes generated by $\pi_k$, using Eq.~\ref{eq:app-gap-kl-identity}.
\end{proof}

\begin{lemma}[Gradient of the functional gap]
\label{lem:app-functional-gradient}
For every fixed prefix $\xb_{<t}$ and token $y_t$,
\begin{align}
    \frac{\partial\Delta(\xb_{<t},r)}{\partial r(\xb_{<t},y_t)}
    &=
    \pi_r(y_t|\xb_{<t})
    \bigl(r(\xb_{<t},y_t)-r^*(\xb_{<t},y_t)\bigr)
    \nonumber\\
    &\quad-
    \pi_r(y_t|\xb_{<t})
    \mathbb E_{y_t'\sim\pi_r(\cdot|\xb_{<t})}
    \left[r(\xb_{<t},y_t')-r^*(\xb_{<t},y_t')\right].
    \label{eq:app-functional-gradient}
\end{align}
\end{lemma}

\begin{proof}
Direct differentiation of Eqs.~\ref{eq:app-partition}--\ref{eq:app-gibbs-policy}, holding the prefix fixed, yields
\begin{align}
    \frac{\partial Z_r(\xb_{<t})}{\partial r(\xb_{<t},y_t)}
    &=\pi^{\mathrm{ref}}(y_t|\xb_{<t})e^{r(\xb_{<t},y_t)},
    \label{eq:app-partition-derivative}\\
    \frac{\partial\pi_r(y_t|\xb_{<t})}{\partial r(\xb_{<t},y_t)}
    &=\pi_r(y_t|\xb_{<t})\bigl(1-\pi_r(y_t|\xb_{<t})\bigr),
    \label{eq:app-policy-derivative-same}\\
    \frac{\partial\pi_r(y_t'|\xb_{<t})}{\partial r(\xb_{<t},y_t)}
    &=-\pi_r(y_t'|\xb_{<t})\pi_r(y_t|\xb_{<t}),
    \qquad y_t'\neq y_t.
    \label{eq:app-policy-derivative-different}
\end{align}
The derivative of $-\log Z_r(\xb_{<t})$ is $-\pi_r(y_t|\xb_{<t})$. Differentiating the expectation in Eq.~\ref{eq:app-functional-gap} produces a direct term $+\pi_r(y_t|\xb_{<t})$, which cancels it. The remaining terms give Eq.~\ref{eq:app-functional-gradient}. No derivative of a prefix distribution is taken.
\end{proof}

Define the uniform optimism event
\begin{equation}
    \mathcal E_{\mathrm{opt}}
    :=
    \left\{
        r_k^+(\xb_{<t},y_t)\geq r^*(\xb_{<t},y_t),
        \ \forall k\in\{0,\ldots,K\},
        \ \forall(\xb_{<t},y_t)\in\mathcal X\times\mathcal Y
    \right\}.
    \label{eq:app-optimism-event}
\end{equation}

\begin{lemma}[Sharp one-step bound under optimism]
\label{lem:app-sharp-one-step}
On $\mathcal E_{\mathrm{opt}}$, for every $k\geq1$,
\begin{equation}
    J(\pi^*)-J(\pi_k)
    \leq
    \mathbb E_{\qb\sim\mathcal D,\,\yb\sim\pi_k(\cdot|\qb)}
    \left[
        \sum_{t=1}^{T}
        \bigl(r_{k-1}^+(\xb_{<t},y_t)-r^*(\xb_{<t},y_t)\bigr)^2
    \right].
    \label{eq:app-sharp-one-step}
\end{equation}
\end{lemma}

\begin{proof}
Fix a prefix $\xb_{<t}$ and write
\begin{equation}
    e(y_t):=r_{k-1}^+(\xb_{<t},y_t)-r^*(\xb_{<t},y_t)\geq0,
    \qquad r_\alpha:=r^*+\alpha e,
    \quad\alpha\in[0,1],
    \label{eq:app-interpolation}
\end{equation}
where the interpolation is only over the token-reward vector at this fixed prefix. Apply the one-dimensional mean-value theorem to $h(\alpha):=\Delta(\xb_{<t},r_\alpha)$. For some $\bar\alpha\in[0,1]$, Lemma~\ref{lem:app-functional-gradient} yields
\begin{align}
    \Delta(\xb_{<t},r_{k-1}^+)-\Delta(\xb_{<t},r^*)
    &=\bar\alpha\left(
        \mathbb E_{\pi_{r_{\bar\alpha}}}[e^2]
        -\mathbb E_{\pi_{r_{\bar\alpha}}}[e]^2
    \right)
    \nonumber\\
    &\leq\bar\alpha\,\mathbb E_{\pi_{r_{\bar\alpha}}}[e^2].
    \label{eq:app-mvt-bound}
\end{align}
Here and below, the token expectations in this proof are conditional on the fixed prefix. To compare the intermediate policy with $\pi_k$, define
\begin{equation}
    m(\alpha):=\alpha\,\mathbb E_{\pi_{r_\alpha}}[e^2].
    \label{eq:app-monotone-function}
\end{equation}
The exponential-family derivative identity gives
\begin{equation}
    m'(\alpha)
    =\mathbb E_{\pi_{r_\alpha}}[e^2]
    +\alpha\left(
        \mathbb E_{\pi_{r_\alpha}}[e^3]
        -\mathbb E_{\pi_{r_\alpha}}[e^2]\,
        \mathbb E_{\pi_{r_\alpha}}[e]
    \right).
    \label{eq:app-monotone-derivative}
\end{equation}
The covariance term is nonnegative because $e\geq0$. Indeed, for independent copies $\xi,\xi'$ of $e(y_t)$ under $\pi_{r_\alpha}$,
\begin{equation}
    \mathbb E[\xi^3]-\mathbb E[\xi^2]\mathbb E[\xi]
    =\frac12\mathbb E\!\left[(\xi-\xi')^2(\xi+\xi')\right]
    \geq0.
    \label{eq:app-positive-covariance}
\end{equation}
Consequently $m'(\alpha)\geq0$, and
\begin{equation}
    \bar\alpha\,\mathbb E_{\pi_{r_{\bar\alpha}}}[e^2]
    =m(\bar\alpha)\leq m(1)
    =\mathbb E_{y_t\sim\pi_k(\cdot|\xb_{<t})}[e(y_t)^2].
    \label{eq:app-intermediate-policy-bound}
\end{equation}
Combining Eqs.~\ref{eq:app-mvt-bound} and~\ref{eq:app-intermediate-policy-bound} proves the sharp contextual-bandit inequality at this prefix. Only after obtaining this pointwise bound do we sum over $t$ and average under the student $\pi_k$. Lemma~\ref{lem:app-objective-decomposition} and conditional expectation then give Eq.~\ref{eq:app-sharp-one-step}. This avoids comparing prefix distributions under different interpolated policies.
\end{proof}

\subsection{Least-Squares Confidence and Uniform Optimism}

\begin{lemma}[Finite-class least-squares confidence]
\label{lem:app-least-squares-confidence}
Under Assumptions~\ref{ass:noise} and~\ref{ass:realizability}, with probability at least $1-\delta/2$, simultaneously for every $k\in\{1,\ldots,K\}$,
\begin{equation}
    \sum_{i=1}^{k}\sum_{u=1}^{T}
    \bigl(\widehat r_k(\xb_{i,<u},y_{i,u})-r^*(\xb_{i,<u},y_{i,u})\bigr)^2
    \leq 8\sigma^2\log\frac{2N_{\mathcal R}K}{\delta}.
    \label{eq:app-empirical-confidence}
\end{equation}
\end{lemma}

\begin{proof}
For $r\in\mathcal R$, write
$D_{i,u}(r):=r(\xb_{i,<u},y_{i,u})-r^*(\xb_{i,<u},y_{i,u})$.
By Eq.~\ref{eq:app-regression-model}, the excess squared loss at a sampled token is
\begin{align}
    &\bigl(r(\xb_{i,<u},y_{i,u})-\ell_{i,u}\bigr)^2
    -\bigl(r^*(\xb_{i,<u},y_{i,u})-\ell_{i,u}\bigr)^2
    \nonumber\\
    &\hspace{3em}=D_{i,u}(r)^2-2D_{i,u}(r)\epsilon_{i,u}.
    \label{eq:app-excess-square-loss}
\end{align}
For $\sigma>0$, conditional sub-Gaussianity implies that, for every fixed $r$, the exponential process obtained by revealing tokens and their feedback in chronological order is a nonnegative supermartingale. At the end of rollout $k$, it equals
\begin{equation}
    \exp\!\left(
        \frac{1}{2\sigma^2}\sum_{i=1}^{k}\sum_{u=1}^{T}
            D_{i,u}(r)\epsilon_{i,u}
        -\frac{1}{8\sigma^2}\sum_{i=1}^{k}\sum_{u=1}^{T}D_{i,u}(r)^2
    \right).
    \label{eq:app-supermartingale}
\end{equation}
The conditional noise assumption is applied after the sampled token is revealed, so adaptivity of the prefix and token does not invalidate this supermartingale. Markov's inequality and a union bound over $r\in\mathcal R$ and the $K$ rollout endpoints yield, with probability at least $1-\delta/2$,
\begin{equation}
    2\sum_{i=1}^{k}\sum_{u=1}^{T}D_{i,u}(r)\epsilon_{i,u}
    \leq
    \frac12\sum_{i=1}^{k}\sum_{u=1}^{T}D_{i,u}(r)^2
    +4\sigma^2\log\frac{2N_{\mathcal R}K}{\delta}
    \label{eq:app-martingale-bound}
\end{equation}
for every such $r$ and $k$. Only $K$ fitting checkpoints are union-bounded, not $KT$ separate estimators. Since $\widehat r_k$ minimizes the empirical loss and $r^*\in\mathcal R$,
\begin{equation}
    \sum_{i=1}^{k}\sum_{u=1}^{T}D_{i,u}(\widehat r_k)^2
    \leq
    2\sum_{i=1}^{k}\sum_{u=1}^{T}D_{i,u}(\widehat r_k)\epsilon_{i,u}.
    \label{eq:app-erm-comparison}
\end{equation}
Substitute $r=\widehat r_k$ into the uniform event in Eq.~\ref{eq:app-martingale-bound} and rearrange to obtain Eq.~\ref{eq:app-empirical-confidence}. For $\sigma=0$, the feedback is noiseless and the same claim follows directly from least-squares optimality.
\end{proof}

\begin{lemma}[Uniform optimism]
\label{lem:app-uniform-optimism}
Under Assumptions~\ref{ass:noise} and~\ref{ass:realizability}, with probability at least $1-\delta/2$, simultaneously for all $k\in\{0,\ldots,K\}$ and all prefix--token pairs,
\begin{equation}
    r^*\in\mathcal C_k,
    \qquad
    0\leq r_k^+(\xb_{<t},y_t)-r^*(\xb_{<t},y_t)
    \leq 2b_k(\xb_{<t},y_t).
    \label{eq:app-two-sided-optimism}
\end{equation}
In particular, $\mathcal E_{\mathrm{opt}}$ holds on this event.
\end{lemma}

\begin{proof}
Work on the event of Lemma~\ref{lem:app-least-squares-confidence}. For $k\geq1$, Eqs.~\ref{eq:app-beta} and~\ref{eq:app-empirical-confidence} imply
\begin{equation}
    \lambda+
    \sum_{i=1}^{k}\sum_{u=1}^{T}
    \bigl(\widehat r_k(\xb_{i,<u},y_{i,u})-r^*(\xb_{i,<u},y_{i,u})\bigr)^2
    \leq\beta^2/2\leq\beta^2.
    \label{eq:app-pointwise-confidence}
\end{equation}
Thus $r^*\in\mathcal C_k$. This also holds for $k=0$, since $\mathcal C_0=\mathcal R$.

For every $r\in\mathcal C_k$, both $r$ and $r^*$ lie in the same empirical ball centered at $\widehat r_k$. Using $(a+b)^2\leq2a^2+2b^2$ gives
\begin{equation}
    \lambda+
    \sum_{i=1}^{k}\sum_{u=1}^{T}
    \bigl(r(\xb_{i,<u},y_{i,u})-r^*(\xb_{i,<u},y_{i,u})\bigr)^2
    \leq\lambda+4(\beta^2-\lambda)\leq4\beta^2.
    \label{eq:app-confidence-diameter}
\end{equation}
The definition of $U_k$ therefore implies
$|r(\xb_{<t},y_t)-r^*(\xb_{<t},y_t)|\leq2\beta U_k(\xb_{<t},y_t)$.
Also, all functions in $\mathcal R$ take values in $[-B,B]$. Taking the pointwise supremum over $\mathcal C_k$ yields
\begin{equation}
    r_k^+(\xb_{<t},y_t)-r^*(\xb_{<t},y_t)
    \leq\min\{2B,\,2\beta U_k(\xb_{<t},y_t)\}
    \leq2b_k(\xb_{<t},y_t).
    \label{eq:app-error-bonus-bound}
\end{equation}
Finally, target membership gives
\begin{equation}
    r_k^+(\xb_{<t},y_t)-r^*(\xb_{<t},y_t)
    =\sup_{r\in\mathcal C_k}r(\xb_{<t},y_t)-r^*(\xb_{<t},y_t)
    \geq0.
    \label{eq:app-optimism-lower}
\end{equation}
Together these prove Eq.~\ref{eq:app-two-sided-optimism}, including the empty-design case.
\end{proof}

\subsection{Proof of Theorem~\ref{thm:main}}

\begin{proof}
For the actual rollout $k$, define the accumulated squared token-reward error
\begin{equation}
    S_k
    :=\sum_{t=1}^{T}
    \bigl(r_{k-1}^+(\xb_{k,<t},y_{k,t})-r^*(\xb_{k,<t},y_{k,t})\bigr)^2,
    \qquad
    \mu_k:=\mathbb E[S_k|\mathcal F_{k-1}].
    \label{eq:app-rollout-error}
\end{equation}
On the event in Lemma~\ref{lem:app-uniform-optimism}, the reverse-KL identity and Lemma~\ref{lem:app-sharp-one-step} imply
\begin{equation}
    \mathrm{Regret}(K)
    \leq\sum_{k=1}^{K}\mu_k,
    \qquad
    S_k\leq4\sum_{t=1}^{T}b_{k-1}(\xb_{k,<t},y_{k,t})^2.
    \label{eq:app-regret-to-bonus}
\end{equation}
Since $\beta\geq2B$, the clipped bonus satisfies
\begin{equation}
    b_{k-1}(\xb_{<t},y_t)^2
    \leq\beta^2\min\{1,\,U_{k-1}(\xb_{<t},y_t)^2\}.
    \label{eq:app-bonus-square}
\end{equation}
For every realized sequence of rollouts, Definition~\ref{def:app-eluder-complexity} gives
\begin{equation}
    \sum_{k=1}^{K}\sum_{t=1}^{T}
    \min\{1,\,U_{k-1}(\xb_{k,<t},y_{k,t})^2\}
    \leq d_E.
    \label{eq:app-eluder-pathwise}
\end{equation}
Consequently, $\sum_{k=1}^{K}S_k\leq4\beta^2d_E$ on the confidence event.

\paragraph{From observed errors to predictable regret.}
A pathwise bound on $\sum_kS_k$ does not by itself bound $\sum_k\mathbb E[S_k|\mathcal F_{k-1}]$ with high probability. We supply the required concentration step while keeping the same logarithmic regret order. Define the pointwise diameter
\begin{equation}
    \omega(\xb_{<t},y_t)
    :=\sup_{r_1,r_2\in\mathcal R}
    |r_1(\xb_{<t},y_t)-r_2(\xb_{<t},y_t)|,
    \qquad c_0:=\max\{\lambda,4B^2\}.
    \label{eq:app-class-diameter}
\end{equation}
The set $\mathcal C_{k-1}$ is nonempty and contained in $\mathcal R$, and $r^*\in\mathcal R$. Thus, even outside the confidence event,
$|r_{k-1}^+-r^*|\leq\omega$ pointwise. Since $U_0=\omega/\sqrt\lambda$ and $\omega\leq2B$,
\begin{equation}
    \omega(\xb_{<t},y_t)^2
    \leq c_0\min\{1,U_0(\xb_{<t},y_t)^2\}.
    \label{eq:app-diameter-width}
\end{equation}
Any valid response can appear as the first rollout in the supremum defining $d_E$. Therefore
\begin{equation}
    0\leq S_k
    \leq c_0\sum_{t=1}^{T}
    \min\{1,U_0(\xb_{k,<t},y_{k,t})^2\}
    \leq c_0d_E=:M
    \label{eq:app-rollout-error-range}
\end{equation}
for every $k$, without conditioning on optimism. If $M=0$, all these errors and the regret are zero. Otherwise, for $z\in[0,1]$, the convexity bound
$e^{-z}\leq1-(1-e^{-1})z$ yields
\begin{equation}
    \mathbb E\!\left[e^{-S_k/M}\mid\mathcal F_{k-1}\right]
    \leq\exp\!\left(-(1-e^{-1})\mu_k/M\right).
    \label{eq:app-predictable-mgf}
\end{equation}
Hence
$\exp\!\left(M^{-1}[(1-e^{-1})\sum_{k=1}^{j}\mu_k-\sum_{k=1}^{j}S_k]\right)$
is a nonnegative supermartingale in $j$. Markov's inequality gives a second event, of probability at least $1-\delta/2$, on which
\begin{equation}
    \sum_{k=1}^{K}\mu_k
    \leq\frac{\sum_{k=1}^{K}S_k+M\log(2/\delta)}{1-e^{-1}}
    \leq2\sum_{k=1}^{K}S_k+2M\log(2/\delta).
    \label{eq:app-predictable-error-bound}
\end{equation}
Intersecting this event with that of Lemma~\ref{lem:app-uniform-optimism} gives probability at least $1-\delta$. On their intersection,
\begin{equation}
    \mathrm{Regret}(K)
    \leq
    \left[8\beta^2+2\max\{\lambda,4B^2\}\log\frac{2}{\delta}\right]d_E.
    \label{eq:app-explicit-regret-bound}
\end{equation}
Substituting Eq.~\ref{eq:app-beta} yields
\begin{equation}
    \mathrm{Regret}(K)
    =\mathcal O\!\left(
        \left[
            B^2+\lambda
            +\sigma^2\log\frac{2N_{\mathcal R}K}{\delta}
            +(B^2+\lambda)\log\frac{2}{\delta}
        \right]d_E
    \right).
    \label{eq:app-general-regret-rate}
\end{equation}
For fixed $B,\sigma,\lambda$, this is
\begin{equation}
    \mathrm{Regret}(K)
    =\mathcal O\!\left(d_E\log\frac{N_{\mathcal R}K}{\delta}\right),
    \label{eq:app-logarithmic-regret-rate}
\end{equation}
as claimed. If $N_{\mathcal R}K=1$, the reward class is a singleton and $d_E=\mathrm{Regret}(K)=0$; otherwise constants inside the logarithm are absorbed into $\mathcal O$. The entire argument estimates immediate token log-ratio rewards and sums fixed-prefix bandit inequalities, rather than propagating errors through an MDP.
\end{proof}

\subsection{Covering-Number Interpretation}
\label{sec:app-covering}

For a token-level policy class $\Pi$, define its log-policy class and induced reward class by
\begin{align}
    \mathcal L_{\Pi}
    &:=\left\{
        (\xb_{<t},y_t)\mapsto\log\pi(y_t|\xb_{<t}):\pi\in\Pi
    \right\},
    \label{eq:app-log-policy-class}\\
    \mathcal R_{\Pi}
    &:=\left\{
        (\xb_{<t},y_t)\mapsto
        \log\frac{\pi(y_t|\xb_{<t})}{\pi^{\mathrm{ref}}(y_t|\xb_{<t})}
        :\pi\in\Pi
    \right\}.
    \label{eq:app-induced-reward-class}
\end{align}
The fixed reference term cancels in every pairwise difference:
\begin{equation}
    \left\|
        \log\frac{\pi_1}{\pi^{\mathrm{ref}}}
        -\log\frac{\pi_2}{\pi^{\mathrm{ref}}}
    \right\|_{\infty}
    =\|\log\pi_1-\log\pi_2\|_{\infty},
    \label{eq:app-covering-metric-equality}
\end{equation}
where the supremum is over prefix--token pairs. Thus these two classes have identical uniform covering numbers. The same cancellation holds in the numerator and denominator of Eq.~\ref{eq:app-uncertainty}, so their token-level uncertainty values and cumulative eluder complexities also coincide.

\paragraph{Finite-cover extension.}
For completeness, the discretization can be made explicit without changing the regret order. Suppose $\mathcal R\subseteq[-B,B]^{\mathcal X\times\mathcal Y}$ admits a finite uniform $\varepsilon$-net with centers in $\mathcal R$, and let
$N_{\mathcal R}=\mathcal N_{\infty}(\varepsilon,\mathcal R)$.
The estimator and confidence set may still be formed over the original class, assuming the empirical minimum is attained. Define
\begin{equation}
    L_\varepsilon:=\log\frac{4N_{\mathcal R}K}{\delta},
    \qquad
    a_\delta:=\sigma\sqrt{2\log\frac{8KT}{\delta}},
    \qquad
    \Gamma_\varepsilon:=4KT\varepsilon(B+a_\delta).
    \label{eq:app-covering-parameters}
\end{equation}
Applying the supermartingale argument to the net centers with failure probability $\delta/4$, and a conditional sub-Gaussian union bound to all observed noises with failure probability $\delta/4$, gives simultaneously
\begin{equation}
    2\sum_{i=1}^{k}\sum_{u=1}^{T}D_{i,u}(r)\epsilon_{i,u}
    \leq\frac12\sum_{i=1}^{k}\sum_{u=1}^{T}D_{i,u}(r)^2
    +4\sigma^2L_\varepsilon+2kT\varepsilon(B+a_\delta)
    \label{eq:app-covering-martingale}
\end{equation}
for all $r\in\mathcal R$ and $k\leq K$, with probability at least $1-\delta/2$. Indeed, replacing $r$ by a net center changes $D_{i,u}(r)$ by at most $\varepsilon$, its square by at most $4B\varepsilon$, and the noise cross-term by at most $2\varepsilon|\epsilon_{i,u}|$, while $|\epsilon_{i,u}|\leq a_\delta$ on the noise event. Least-squares optimality then yields
\begin{equation}
    \sum_{i=1}^{k}\sum_{u=1}^{T}D_{i,u}(\widehat r_k)^2
    \leq8\sigma^2L_\varepsilon+\Gamma_\varepsilon.
    \label{eq:app-covering-confidence}
\end{equation}
Consequently, using
\begin{equation}
    \beta^2
    :=\max\left\{
        4B^2,\,
        16\sigma^2L_\varepsilon+2\Gamma_\varepsilon+2\lambda
    \right\}
    \label{eq:app-covering-radius}
\end{equation}
in Algorithm~\ref{alg:rpd-theory} proves the same optimism and explicit regret bound as before. Choosing
$0<\varepsilon\leq[KT(1+B+a_\delta)]^{-1}$
ensures $\Gamma_\varepsilon\leq4$. Thus Theorem~\ref{thm:main} retains its logarithmic form with the indicated covering number, while any additional dependence on $K,T$ through the covering resolution and $d_E$ remains explicit. The covering number is therefore evaluated on prefix--token pairs at the stated resolution.

\subsection{Comparison with On-Policy Distillation under Noisy Expert Feedback}
\label{sec:app-noisy-expert-comparison}

Building on the discussion following Assumption~\ref{ass:noise}, we compare Theorem~\ref{thm:main} with the unknown-corruption guarantee of \citet{sriraman2026behavior}. We specialize their sequential formulation to token generation: states are prefixes $\xb_{<t}=(\qb,\yb_{<t})$, actions are tokens $y_t\in\mathcal Y$, the horizon is $T$, and $K$ counts rollout rounds. Write $P^\pi$ for the joint distribution of $\qb\sim\mathcal D$ and $\yb\sim\pi(\cdot|\qb)$. Their error criterion is the squared Hellinger distance, with normalization
\[
    D_{\mathrm H^2}(P^\pi,P^{\pi'})
    :=\mathbb E_{\qb\sim\mathcal D}\!\left[
        1-\sum_{\yb}\sqrt{\pi(\yb|\qb)\pi'(\yb|\qb)}
    \right].
\]

\begin{theorem}[NAILGUN with unknown corruption, {\citealp[Theorems~9 and~15]{sriraman2026behavior}}]
\label{thm:app-nailgun}
Let $\Pi$ be a finite class of deterministic token-level policies containing the clean expert $\pi^*$. Suppose the observed teacher satisfies
\begin{equation}
    \pi^E(y_t|\xb_{<t})
    =(1-\zeta)\pi^*(y_t|\xb_{<t})
      +\zeta\nu_t(y_t|\xb_{<t}),
    \label{eq:app-nailgun-corruption}
\end{equation}
where $0\leq\zeta\leq\alpha$, $0<\alpha,\rho<1$, and
$\nu_t(y_t|\xb_{<t})\leq\rho$ for every prefix--token pair. Assume
$\gamma:=1-\alpha(1+\rho)>0$.
Given the bounds $\alpha$ and $\rho$, but without knowing $\zeta$ or $\nu$, NAILGUN uses sampled teacher-token labels at learner-visited prefixes and returns a uniformly selected rollout policy $\widehat\pi$ after $K$ rounds such that
\begin{equation}
    \mathbb E\!\left[
        D_{\mathrm H^2}(P^{\pi^*},P^{\widehat\pi})
    \right]
    \leq
    \frac{4\log|\Pi|}
    {K\bigl(\sqrt{1-\alpha}-\sqrt{\alpha\rho}\bigr)^2}
    \leq
    \frac{8\log|\Pi|}{K\gamma^2}.
    \label{eq:app-nailgun-rate}
\end{equation}
The expectation includes the training randomness and the selection of $\widehat\pi$.
\end{theorem}

\begin{remark}[Why vanilla OPD is insufficient under corruption]
\label{rem:app-noisy-opd}
NAILGUN is a noise-aware algorithm, not vanilla reverse-KL distillation: directly matching the corrupted teacher need not recover the clean expert~\citep{sriraman2026behavior}. Under Assumption~\ref{ass:noise}, our feedback is instead conditionally unbiased. Theorem~\ref{thm:main} establishes a finite-sample guarantee for optimistic least-squares learning, rather than proving that vanilla OPD must fail under our assumptions.
\end{remark}

\begin{remark}[On-policy interaction versus optimistic exploration]
\label{rem:app-exploration-comparison}
NAILGUN updates a weighted aggregate of deterministic policies using noisy expert labels~\citep[Algorithm~3]{sriraman2026behavior}. In contrast, Algorithm~\ref{alg:rpd-theory} combines optimistic token selection with least-squares estimation of log-ratio rewards. Our analysis shows how this explicit exploration mechanism yields fast learning with general reward-function approximation; ordinary on-policy sampling alone does not provide the same guarantee.
\end{remark}

\begin{remark}[Different noise dependence and identifiability]
\label{rem:app-noise-dependence}
NAILGUN's bound depends inversely on the squared corruption margin $\gamma^2$, reflecting the difficulty of distinguishing clean and corrupted labels. Our bound instead depends on $\sigma^2$ through estimation uncertainty in Eq.~\ref{eq:app-general-regret-rate}. As discussed following Assumption~\ref{ass:noise}, this difference relies on conditionally unbiased feedback, not robustness to arbitrary corruption. Our analysis explicitly characterizes how feedback uncertainty affects distillation under this noise model.
\end{remark}

\begin{remark}[Comparing rates, support, and rollout budgets]
\label{rem:app-rate-comparison}
NAILGUN bounds expected final-policy Hellinger error, whereas Theorem~\ref{thm:main} gives high-probability logarithmic cumulative reverse-KL regret and a $\widetilde{\mathcal O}(K^{-1})$ guarantee for the response-level mixture policy. Both have a $1/K$ output-error factor up to complexity and logarithmic terms, with $K$ counting rollouts. Our analysis directly controls the distillation objective throughout training and accommodates stochastic targets under bounded log-ratio rewards. Different support assumptions and horizon-dependent complexities make these complementary guarantees, rather than a strict improvement in sample complexity.
\end{remark}

\section{Generated Samples}

In this section, we present representative samples generated on the AMC23, AIME24, and AIME25 benchmarks using our proposed LSPD and LSPD-RB methods, under the Qwen3-8B non-thinking teacher to Qwen3-4B-Base student distillation setup.

% Requires amsmath, amssymb, tcolorbox, and the user-provided rollout environment.
% These are complete saved responses; only Markdown/LaTeX formatting is changed.
% Add to the preamble: \tcbuselibrary{skins,breakable}
% Optional bold typewriter font support: \usepackage{courier}

\begingroup
\tcbset{
  enhanced,
  breakable,
  fontupper=\ttfamily\raggedright,
  fontlower=\ttfamily\raggedright,
  fonttitle=\footnotesize\ttfamily\bfseries,
}

% LSPD, checkpoint 160; AMC23, dataset index 45; evaluation line 1211.
% Requires amsmath, amssymb, tcolorbox, and the user-provided rollout environment.
% These are complete saved responses; only Markdown/LaTeX formatting is changed.
% Add to the preamble: \tcbuselibrary{skins,breakable}
% Optional bold typewriter font support: \usepackage{courier}

\begingroup
\tcbset{
  enhanced,
  breakable,
  fontupper=\ttfamily\raggedright,
  fontlower=\ttfamily\raggedright,
  fonttitle=\normalfont\footnotesize\bfseries,
}

% LSPD, checkpoint 160; AMC23, dataset index 45; evaluation line 1211.
\begin{rollout}{LSPD Solutions in AMC23}
\textbf{Question.} What is the degree measure of the acute angle formed by lines with slopes $2$ and $\frac{1}{3}$?
\tcblower
\textbf{Solution.}

We are given two lines with slopes \(m_1 = 2\) and \(m_2 = \frac{1}{3}\), and we are asked to find the \textbf{acute angle} formed between them.

\textbf{Step 1: Use the formula for the angle between two lines}

The formula to find the angle \(\theta\) between two lines with slopes \(m_1\) and \(m_2\) is:

\[
\tan(\theta) = \left| \frac{m_1 - m_2}{1 + m_1 m_2} \right|
\]

This gives the \textbf{acute angle} between the lines.

\textbf{Step 2: Plug in the values}

We have:

\begin{itemize}
\item
  \(m_1 = 2\)
\item
  \(m_2 = \frac{1}{3}\)
\end{itemize}

Compute:

\[
\tan(\theta) = \left| \frac{2 - \frac{1}{3}}{1 + 2 \cdot \frac{1}{3}} \right|
= \left| \frac{\frac{6}{3} - \frac{1}{3}}{1 + \frac{2}{3}} \right|
= \left| \frac{\frac{5}{3}}{\frac{5}{3}} \right|
= 1
\]

So:

\[
\tan(\theta) = 1
\]

\textbf{Step 3: Find the angle}

\[
\theta = \tan^{-1}(1) = 45^\circ
\]

\textbf{Final Answer:}

\[
\boxed{45^\circ}
\]
\end{rollout}

% LSPD, checkpoint 160; AIME24, dataset index 1; evaluation line 1832.
\begin{rollout}{LSPD Solutions in AIME24}
\textbf{Question.} There exist real numbers $x$ and $y$, both greater than 1, such that $\log_x\left(y^x\right)=\log_y\left(x^{4y}\right)=10$. Find $xy$.
\tcblower
\textbf{Solution.}

We are given the equations:

\[
\log_x(y^x) = 10 \quad \text{and} \quad \log_y(x^{4y}) = 10
\]

We are to find \(xy\), where both \(x\) and \(y\) are \textbf{real numbers greater than 1}.

\textbf{Step 1: Use logarithmic identities}

Recall the logarithmic identity:

\[
\log_b(a^c) = c \cdot \log_b(a)
\]

So we can rewrite the given equations using this identity.

\textbf{First equation:} \[
\log_x(y^x) = x \cdot \log_x(y) = 10 \quad \text{(Equation 1)}
\]

\textbf{Second equation:} \[
\log_y(x^{4y}) = 4y \cdot \log_y(x) = 10 \quad \text{(Equation 2)}
\]

Let's denote:

\[
\log_x(y) = a \quad \text{and} \quad \log_y(x) = b
\]

Then from Equation 1:

\[
x \cdot a = 10 \Rightarrow a = \frac{10}{x}
\]

From Equation 2:

\[
4y \cdot b = 10 \Rightarrow b = \frac{10}{4y} = \frac{5}{2y}
\]

Now, recall the identity:

\[
\log_x(y) = \frac{1}{\log_y(x)} \Rightarrow a = \frac{1}{b}
\]

Substitute the expressions for \(a\) and \(b\):

\[
\frac{10}{x} = \frac{1}{\frac{5}{2y}} = \frac{2y}{5}
\]

\textbf{Step 2: Solve the equation}

\[
\frac{10}{x} = \frac{2y}{5}
\]

Multiply both sides by \(5x\):

\[
50 = 2xy \Rightarrow xy = \boxed{25}
\]

\textbf{Final Answer:}

\[
\boxed{25}
\]
\end{rollout}

% LSPD, checkpoint 160; AIME25, dataset index 0; evaluation line 5.
\begin{rollout}{LSPD Solutions in AIME25}
\textbf{Question.} Find the sum of all integer bases $b>9$ for which $17_{b}$ is a divisor of $97_{b}$.
\tcblower
\textbf{Solution.}

We are given the problem:\\
\textbf{Find the sum of all integer bases \(b > 9\) for which \(17_b\) is a divisor of \(97_b\).}

\textbf{Step 1: Convert the numbers from base \(b\) to base 10}

We are given:

\begin{itemize}
\item
  \(17_b\) is the number \(1 \cdot b + 7 = b + 7\)
\item
  \(97_b\) is the number \(9 \cdot b + 7 = 9b + 7\)
\end{itemize}

So the problem becomes:\\
Find all integers \(b > 9\) such that \(b + 7\) divides \(9b + 7\), and then sum those values of \(b\).

\textbf{Step 2: Set up the divisibility condition}

We want: \[
b + 7 |9b + 7
\]

This means: \[
9b + 7 \equiv 0 \pmod{b + 7}
\]

Let's perform the division: \[
9b + 7 = 9(b + 7) - 63 + 7 = 9(b + 7) - 56
\]

So: \[
9b + 7 \equiv -56 \pmod{b + 7}
\]

Therefore: \[
b + 7 |56
\]

\textbf{Step 3: Find all divisors of 56}

We need to find all integers \(d\) such that: \[
b + 7 = d \Rightarrow b = d - 7
\]

So, we list the \textbf{positive divisors} of 56: \[
\text{Divisors of } 56: \{1, 2, 4, 7, 8, 14, 28, 56\}
\]

Now, we want \(b > 9\), so \(d - 7 > 9 \Rightarrow d > 16\).

Check which of these divisors are greater than 16:

\begin{itemize}
\item
  \(1\): too small
\item
  \(2\): too small
\item
  \(4\): too small
\item
  \(7\): too small
\item
  \(8\): too small
\item
  \(14\): \(b = 14 - 7 = 7\) \(\to\) too small
\item
  \(28\): \(b = 28 - 7 = 21\) \(\to\) OK
\item
  \(56\): \(b = 56 - 7 = 49\) \(\to\) OK
\end{itemize}

So the valid values of \(b\) are: \[
b = 21, 49
\]

\textbf{Step 4: Compute the sum}

\[
21 + 49 = \boxed{70}
\]

\textbf{Final Answer:} \[
\boxed{70}
\]
\end{rollout}

% LSPD-RB, checkpoint 10; AMC23, dataset index 45; evaluation line 1206.
\begin{rollout}{LSPD-RB Solutions in AMC23}
\textbf{Question.} What is the degree measure of the acute angle formed by lines with slopes $2$ and $\frac{1}{3}$?
\tcblower
\textbf{Solution.}

We are given two lines with slopes \(m_1 = 2\) and \(m_2 = \frac{1}{3}\), and we are asked to find the \textbf{acute angle} formed between them.

\textbf{Step 1: Use the formula for the angle between two lines}

The formula for the angle \(\theta\) between two lines with slopes \(m_1\) and \(m_2\) is:

\[
\tan \theta = \left| \frac{m_2 - m_1}{1 + m_1 m_2} \right|
\]

We are interested in the \textbf{acute angle}, so we take the \textbf{absolute value} of the right-hand side.

\textbf{Step 2: Plug in the values}

\[
\tan \theta = \left| \frac{\frac{1}{3} - 2}{1 + 2 \cdot \frac{1}{3}} \right|
\]

Simplify the numerator and denominator:

\begin{itemize}
\item
  Numerator: \(\frac{1}{3} - 2 = \frac{1}{3} - \frac{6}{3} = -\frac{5}{3}\)
\item
  Denominator: \(1 + \frac{2}{3} = \frac{5}{3}\)
\end{itemize}

So,

\[
\tan \theta = \left| \frac{-\frac{5}{3}}{\frac{5}{3}} \right| = \left| -1 \right| = 1
\]

\textbf{Step 3: Find the angle}

\[
\tan \theta = 1 \Rightarrow \theta = \tan^{-1}(1) = 45^\circ
\]

Since the angle is acute (less than \(90^\circ\)), the answer is:

\[
\boxed{45^\circ}
\]

\textbf{Final Answer:}

\[
\boxed{45^\circ}
\]
\end{rollout}

% LSPD-RB, checkpoint 10; AIME24, dataset index 1; evaluation line 1836.
\begin{rollout}{LSPD-RB Solutions in AIME24}
\textbf{Question.} There exist real numbers $x$ and $y$, both greater than 1, such that $\log_x\left(y^x\right)=\log_y\left(x^{4y}\right)=10$. Find $xy$.
\tcblower
\textbf{Solution.}

We are given the following two equations involving logarithms:

\[
\log_x(y^x) = 10 \quad \text{and} \quad \log_y(x^{4y}) = 10
\]

We are to find the value of \(xy\), given that \(x, y > 1\).

\textbf{Step 1: Simplify the first equation}

\[
\log_x(y^x) = 10
\]

Using the logarithm power rule:

\[
\log_x(y^x) = x \cdot \log_x(y) = 10
\]

So,

\[
x \cdot \log_x(y) = 10 \tag{1}
\]

\textbf{Step 2: Simplify the second equation}

\[
\log_y(x^{4y}) = 10
\]

Again using the logarithm power rule:

\[
\log_y(x^{4y}) = 4y \cdot \log_y(x) = 10
\]

So,

\[
4y \cdot \log_y(x) = 10 \tag{2}
\]

\textbf{Step 3: Express \(\log_x(y)\) in terms of \(\log_y(x)\)}

We recall the change of base formula:

\[
\log_x(y) = \frac{1}{\log_y(x)}
\]

Let's denote:

\[
\log_y(x) = a \quad \Rightarrow \quad \log_x(y) = \frac{1}{a}
\]

Now substitute into equations (1) and (2):

\begin{itemize}
\item
  From (1): \(x \cdot \frac{1}{a} = 10 \Rightarrow x = 10a\)
\item
  From (2): \(4y \cdot a = 10 \Rightarrow y = \frac{10}{4a} = \frac{5}{2a}\)
\end{itemize}

\textbf{Step 4: Substitute into \(xy\)}

We now have expressions for \(x\) and \(y\) in terms of \(a\):

\[
x = 10a, \quad y = \frac{5}{2a}
\]

So,

\[
xy = (10a) \cdot \left( \frac{5}{2a} \right) = \frac{50a}{2a} = 25
\]

\textbf{Step 5: Verify that \(x, y > 1\)}

We have:

\begin{itemize}
\item
  \(x = 10a\)
\item
  \(y = \frac{5}{2a}\)
\end{itemize}

We want \(x, y > 1\), so:

\begin{itemize}
\item
  \(10a > 1 \Rightarrow a > \frac{1}{10}\)
\item
  \(\frac{5}{2a} > 1 \Rightarrow a < \frac{5}{2}\)
\end{itemize}

So, \(a \in \left( \frac{1}{10}, \frac{5}{2} \right)\), which is valid.

\textbf{Final Answer:}

\[
\boxed{25}
\]
\end{rollout}

% LSPD-RB, checkpoint 10; AIME25, dataset index 0; evaluation line 13.
\begin{rollout}{LSPD-RB Solutions in AIME25}
\textbf{Question.} Find the sum of all integer bases $b>9$ for which $17_{b}$ is a divisor of $97_{b}$.
\tcblower
\textbf{Solution.}

We are given that \(17_{b}\) divides \(97_{b}\), and we are to find the \textbf{sum of all integer bases \(b > 9\)} for which this is true.

\textbf{Step 1: Convert both numbers from base \(b\) to base 10}

We convert \(17_{b}\) and \(97_{b}\) to base 10:

\begin{itemize}
\item
  \(17_{b} = 1 \cdot b + 7 = b + 7\)
\item
  \(97_{b} = 9 \cdot b + 7 = 9b + 7\)
\end{itemize}

So, the condition becomes:

\[
b + 7 |9b + 7
\]

\textbf{Step 2: Use divisibility condition}

We are told that \(b + 7\) divides \(9b + 7\), so:

\[
b + 7 |9b + 7
\]

To simplify this, we use the fact that if \(d |a\) and \(d |b\), then \(d |(a - kb)\) for any integer \(k\).

Let's subtract \(9 \cdot (b + 7)\) from \(9b + 7\):

\[
(9b + 7) - 9(b + 7) = 9b + 7 - 9b - 63 = -56
\]

So:

\[
b + 7 |-56
\]

Since \(b + 7\) divides \(-56\), it also divides \(56\). So we are looking for all \textbf{positive integers \(b > 9\)} such that:

\[
b + 7 |56
\]

\textbf{Step 3: Find all divisors of 56}

List all positive divisors of 56:

\[
\text{Divisors of 56: } 1, 2, 4, 7, 8, 14, 28, 56
\]

Now, we want \(b + 7\) to be one of these, and \(b > 9\), so:

\[
b = d - 7 \quad \text{and} \quad b > 9 \Rightarrow d - 7 > 9 \Rightarrow d > 16
\]

So we filter the divisors of 56 that are greater than 16:

\[
\text{Divisors of 56 > 16: } 28, 56
\]

So possible values of \(b + 7\) are 28 and 56, and thus:

\begin{itemize}
\item
  \(b = 28 - 7 = 21\)
\item
  \(b = 56 - 7 = 49\)
\end{itemize}

Both values are greater than 9, so they are valid.

\textbf{Step 4: Final Answer}

We now sum all such \(b\):

\[
21 + 49 = \boxed{70}
\]

\textbf{Bonus: Check that these work}

\begin{itemize}
\item
  For \(b = 21\):

  \begin{itemize}
    \item
    \(17_{21} = 21 + 7 = 28\)
  \item
    \(97_{21} = 9 \cdot 21 + 7 = 189 + 7 = 196\)
  \item
    \(28 |196\) \(\to\) True (196 \(\div\) 28 = 7)
  \end{itemize}
\item
  For \(b = 49\):

  \begin{itemize}
    \item
    \(17_{49} = 49 + 7 = 56\)
  \item
    \(97_{49} = 9 \cdot 49 + 7 = 441 + 7 = 448\)
  \item
    \(56 |448\) \(\to\) True (448 \(\div\) 56 = 8)
  \end{itemize}
\end{itemize}

Both values satisfy the condition.

\textbf{Final Answer:}

\[
\boxed{70}
\]
\end{rollout}

\endgroup

\end{document}